\documentclass[11pt]{article}
\usepackage[sort&compress,square,comma,authoryear]{natbib}
\usepackage[a4paper, total={6.1in, 10in}]{geometry}
\usepackage{enumitem}
\usepackage[utf8]{inputenc} 
\usepackage[T1]{fontenc}    
\usepackage{hyperref}       
\usepackage{url}            
\usepackage{booktabs}  
\usepackage{titletoc}
\usepackage{tabularx}
\usepackage{float}
\usepackage{multirow}
\newcolumntype{Y}{>{\centering\arraybackslash}X}
\newcolumntype{Z}{>{\raggedleft\arraybackslash}X}
\usepackage[parfill]{parskip}
\usepackage{amsfonts}       
\usepackage{algorithm}
\usepackage{algpseudocode}
\usepackage{amsmath}
\usepackage{nicefrac}       
\usepackage{microtype}      
\usepackage[dvipsnames]{xcolor}         
\usepackage{amsmath,amssymb,amsthm}
\usepackage{bbold}
\usepackage{colortbl}
\usepackage{graphicx} 
\usepackage{mathrsfs,mathtools}
\usepackage{subfigure}
\usepackage{thm-restate}
\usepackage{wrapfig}
\usepackage{caption}
\usepackage[capitalize]{cleveref}
\usepackage{tikz}
\usetikzlibrary{calc, fadings, decorations, shapes, positioning, arrows}
\usepackage[graphicx]{realboxes}
\usepackage{pifont}
\usepackage{tcolorbox}

\newcommand{\best}[1]{\textbf{#1}}
\newcommand{\rhat}{\hat{r}}
\newcommand{\X}{\mathcal{X}}

\definecolor{salmon}{RGB}{234,153,153}
\definecolor{cornflowerblue}{RGB}{100,149,237}
\hypersetup{
    colorlinks,
    linkcolor={cornflowerblue},
    citecolor={cornflowerblue},
    urlcolor={salmon}
}

\definecolor{darkgreen}{rgb}{0.0, 0.5, 0.0}
\usepackage[textsize=tiny]{todonotes}
\usepackage{tabularx}
\usepackage{wrapfig}
\definecolor{darkblue}{rgb}{0, 0, 0.5}
\hypersetup{colorlinks=true, citecolor=darkblue, linkcolor=darkblue, urlcolor=darkblue}
\usepackage{tikz}

\theoremstyle{plain}
\newtheorem{theorem}{Theorem}[section]
\newtheorem{proposition}[theorem]{Proposition}
\newtheorem{lemma}[theorem]{Lemma}

\theoremstyle{definition}

\theoremstyle{remark}
\newtheorem{remark}[theorem]{Remark}

\newcommand{\midsepremove}{\aboverulesep = 0mm \belowrulesep = 0mm}
\midsepremove


\title{Learning to Orchestrate Agents under Uncertainty}

\date{}

\usepackage{authblk}

\author[1]{Mary Chriselda Antony Oliver\thanks{Correspondence to: \texttt{mca52@cam.ac.uk}. \textit{${}^{\dagger}$indicates joint supervision.}}}
\author[2]{Lan Jiang}
\author[3]{Aaron Bundi Anampiu}
\author[2]{Elaf Almahmoud}
\author[4]{Francesco Quinzan${}^{\dagger}$}
\author[2]{Umang Bhatt${}^{\dagger}$}

\affil[1]{Department of Applied Mathematics and Theoretical Physics,  University of Cambridge}
\affil[2]{Centre for Human-Inspired Artificial Intelligence, University of Cambridge}
\affil[3]{African Institute for Mathematical Sciences, South Africa}
\affil[4]{Department of Engineering Science, University of Oxford}

\begin{document}

\maketitle

\begin{abstract}
\noindent
Adaptive orchestration of heterogeneous agents requires making sequential delegation decisions under uncertain and evolving agent behaviour, e.g., coordinating specialised AI models with varying reliability, cost, and response quality. While prior work on agent orchestration focuses on performance or cost, uncertainty in agent reliability and output distributions is typically not modelled explicitly at the orchestration level. In this work, we study the problem of adaptive orchestration of heterogeneous agents under uncertainty, where a meta-controller must decide when to delegate to an agent, accounting for reliability, cost, and uncertainty. We propose BOT-Orch, a lightweight framework that recasts orchestration as a bandit problem over agents, regularized by OT distances between agent output distributions and task-specific reference distributions. We show that the regularised orchestration enjoys $\mathcal{O}(\sqrt{T})$ regret under standard assumptions, and provably induces preference ordering among agents with identical mean rewards but differing distributional alignment. Empirically, we demonstrate that BOT-Orch outperforms standard bandit and heuristic baselines in synthetic but adversarial task allocation settings with heterogeneous, non-i.i.d. agent behaviour.

\end{abstract}

\section{Introduction}

Incorporating uncertainty in delegated decision-making process is a fundamental challenge across machine learning and autonomous systems~\citep{WurmanDAndreaMountz2007,OlfatiSaberFaxMurray2007,shalev2012online}. In real-world environments, where stochasticity is dominant (e.g., partial and noisy observations), variability in agent capabilities makes centralized control impractical and can break classical solutions. Classical frameworks (e.g., Dec-POMDPs, stochastic games, DCOPs) explicitly model uncertainty in decision-making but often assume known system models, have limited scalability, or simplify agent interactions, restricting applicability in complex coordination settings~\citep{bernstein2002complexity,OliehoekAmato2016DecPOMDP,hansen2004dynamic,heifetz2006interactive,shoham2008multiagent,fioretto2018distributed}. In practice, heterogeneity in agent reliability, operational cost, and resource constraints further complicates coordination, especially at scale where performance, risk, and cost must be jointly managed~\citep{RizkAwadTunstel2019CooperativeHeterogeneous,Arjun2025Survey}. Surveys in distributed decision-making and heterogeneous multi-agent coordination consistently identify these issues as key barriers to scalable deployment~\citep{OlfatiSaberFaxMurray2007,RizkAwadTunstel2019CooperativeHeterogeneous}.

Multi-agent (including single-agent) reinforcement learning (MARL) provides tools for learning cooperative behaviors under partial observability and decentralised execution~\citep{Busoniu2008MARLSurvey,HernandezLeal2019Survey}. Value decomposition and centralized training with decentralized policies can maximize joint rewards even with distinct roles and observations~\citep{Sunehag2018VDN,Rashid2018QMIX}. However, many MARL methods implicitly assume similar reliability and reaction costs, and rarely address adaptive selection among heterogeneous agents of uncertain outputs.

Modern AI systems and single-agent orchestration frameworks mitigate these challenges. In planning, reasoning, and AI agent teams, systems increasingly rely on specialized agents with different expertise, reliability, and computational cost. Effective orchestration requires deciding which agent(s) to invoke and when to combine multiple predictions for robustness. Recent work shows that dynamic selection and composition conditioned on context can outperform static pipelines~\citep{Park2023GenerativeAgents,Liang2023AgentSystems,Cheng2024LLMAgents}.

A natural abstraction for sequential decision-making under uncertainty is the multi-armed bandit (MAB) framework, capturing the exploration--exploitation trade-off~\citep{lattimore2020bandit,bubeck2012regret}. In the special case, single-agent extensions of the MAB framework study cooperation, regret minimization, and communication among learners~\citep{landgren2016distributed,gupta2021multi}. Yet most bandit models treat agents as interchangeable up to mean reward and ignore query costs, limiting their suitability for orchestration with heterogeneous reliability, non-stationarity, and explicit invocation costs.

Optimal Transport (OT) offers a complementary way to compare distributions and quantify discrepancies between uncertain outcomes~\citep{villani2003topics,peyre2019computational}. OT is widely used for uncertainty-aware distributional comparison in domain adaptation~\citep{courty2016optimal}, generative modeling~\citep{arjovsky2017wasserstein,bousquet2017optimal}, and statistical inference~\citep{panaretos2020statistical}. However, OT remains underexplored in single-agent delegation as a mechanism to compare agent output distributions and guide adaptive orchestration. Recent work suggests OT can be fruitfully combined with multi-agent reinforcement learning for scalability and alignment in complex environments~\citep{baheri2024synergy}.

\noindent\textbf{Our contribution.}
\begin{enumerate}[label={$\bullet$},leftmargin=*,itemsep=1pt,topsep=1pt]

\item \textbf{Bandit-based orchestration with OT alignment:} 
We cast delegation over a single agent as a stochastic bandit regularized by OT distances between agent output distributions and task-specific references, enabling uncertainty-aware alignment and adaptive decision-making.

\item \textbf{Theoretical guarantees:} 
We establish sublinear OT-regularized regret (\ref{item:sublinera_regret} of Theorem \ref{thm:main_ot_bandits}), structural optimality and robustness to noisy alignment (\ref{item:structure_OT_optimality}-\ref{item:robustness} of Theorem \ref{thm:main_ot_bandits}), and convergence and consistency properties (\ref{item:consistency_weights}-\ref{item:convergence_weights} of Theorem \ref{thm:main_ot_bandits}).

\item \textbf{Empirical validation under heterogeneity and shift:} 
We evaluate BOT-Orch across synthetic and semi-synthetic settings, including a human--AI triage scenario under deployment shift (Section~\ref{suntehtic_experiments}-\ref{sec:real_world_experiments}), demonstrating consistent improvements over standard bandit and heuristic baselines in both i.i.d.\ and non-i.i.d.\ environments.

\end{enumerate}

\section{Related Work}
\noindent\textbf{Agent orchestration.} 
Prior work on agent orchestration largely focuses on selecting agents based on expected utility or offline accuracy \citep{keswani2021towards,lai2022human,rasal2024navigating}. However, these approaches often overlook practical constraints such as availability, cost, and capability. Recent work on human–AI orchestration highlights how inter-agent interactions shape system-level decision networks \citep{collins2024modulating}, yet uncertainty-aware adaptive orchestration under realistic constraints remains underexplored. We address this gap via uncertainty-aware OT-based orchestration. In a related direction, DiscoPOP \citep{discopop2025} learns loss functions without human input, optimizing over objectives. This suggests that automatic objective discovery could extend to orchestration.

\noindent\textbf{Bandit approaches.}
Multi-armed bandits provide a natural framework for sequential decision making under uncertainty, balancing exploration and exploitation \citep{lattimore2020bandit,chen2024survey_contextual_bandits,tong2024survey_MAB,bandit_review2024}. Recent surveys summarize advances in classical and contextual bandits and their applications \citep{chen2024survey_contextual_bandits,bandit_review2024}, while empirical studies highlight design choices affecting performance \citep{bietti2021contextual_bakeoff}. Emerging work connects bandits with large language models in complex environments \citep{xie2026LLM_bandits_survey}. Here, we use bandits to learn orchestrations that maximize expected utility across heterogeneous agents with uncertain performance.

\noindent\textbf{OT for uncertainty-aware orchestration.}
Optimal transport provides a principled framework for comparing probability distributions via geometrically meaningful discrepancies \citep{villani2008optimal}. We use OT to quantify uncertainty in heterogeneous agent outputs through distributional disagreement and variability, building on recent work in OT-based uncertainty quantification with applications to health data science \citep{oliver2025uncertainty,oliver2025laplace,oliver2025conic}. This provides uncertainty-aware weights capturing the trade-off between confidence and disagreement. 
The framework is well-suited for distributional and uncertainty-aware reinforcement learning settings such as those in \citep{osband2013more,bellemare2017distributional}, offering a unified approach to adaptive multi-agent coordination with both theoretical guarantees and strong empirical performance.

\section{Preliminaries and Notation}
In this section, we introduce the relevant concepts that are foundational to our framework.

\subsection{Agents, Tasks, and Task Space}\label{subsec:agents_tasks}

Let $\mathcal{A} := \{a_1,\dots,a_M\}$ denote a finite set of $M$ agents. Time evolves over a finite horizon $t=1,\dots,T$. At each round $t$, a task $x_t \in \mathcal{X}$ is observed. We assume that tasks are generated by a history-dependent stochastic process $x_t \sim \mathcal{P}(\cdot \mid \mathcal{H}_t)$,
where the history filtration is defined as
\[\mathcal{H}_t := \sigma\big( x_s, A_s, \mathbf{R}_s, \mathbf{W}_s : s=1,\dots,t-1 \big),
\]
and $A_s \in \{1,\dots,M\}$ is $\mathcal{H}_s$-measurable and denotes the agent selected at time $s$. For each agent $a_i \in \mathcal{A}$ and task $x_t$, let $R_t^i \in \mathbb{R}$ denote the (possibly counterfactual) reward that would be obtained if $x_t$ were assigned to agent $a_i$, and let $W_t^i \in \mathbb{R}$ denote the corresponding task–agent alignment cost. Define vectors
$\mathbf{R}_t := (R_t^1,\dots,R_t^M)^\top \in \mathbb{R}^M$, $\mathbf{W}_t := (W_t^1,\dots,W_t^M)^\top \in \mathbb{R}^M$.
Analogously, we define the corresponding conditional mean quantities:
$r^i(x_t) := \mathbb{E}[R_t^i \mid x_t]$, $w^i(x_t) := \mathbb{E}[W_t^i \mid x_t]$. Note that task arrivals may depend on past observations, environmental conditions, or previous assignments.

\subsection{Correlated Multi-Agent Rewards}
\label{sec:correlated_rewards}

Let $\mathbf{R}_t \in \mathbb{R}^M$ denote the vector of (joint) rewards at round $t$. We model its conditional moments given the current task $x_t$ and history $\mathcal{H}_t$. Assume there exists a function $\mathbf{r}$ and a history-dependent drift term $\mathbf{f}:\mathcal{H}_t\to\mathbb{R}^M$ such that
\begin{equation*}
\mathbb{E}[\mathbf{R}_t \mid x_t, \mathcal{H}_t]
= \mathbf{r}(x_t) + \mathbf{f}(\mathcal{H}_t),
\end{equation*}
and a positive semidefinite matrix $\Sigma_t$ satisfying
\begin{equation*}
\mathrm{Cov}[\mathbf{R}_t \mid x_t, \mathcal{H}_t] = \Sigma_t.
\end{equation*}
Here, $\Sigma_t$ may vary over time, capturing non-stationarity and cross-agent correlations.

\paragraph{I.I.D. tasks as a special case.}
Suppose $x_t \stackrel{\text{i.i.d.}}{\sim} \mathcal{P}_X$. If, in addition,  the reward process is conditionally independent of the past given the current task, i.e.
$\mathbf{R}_t \perp \mathcal{H}_t \mid x_t$,
then $\mathbf{f} \equiv 0$, and the model reduces to
\[
\mathbb{E}[\mathbf{R}_t \mid x_t] = \mathbf{r}(x_t), 
\qquad
\mathrm{Cov}[\mathbf{R}_t \mid x_t] = \Sigma,
\]
for a constant covariance matrix $\Sigma$.

\subsection{Survival-Based Rewards with Latent Frailty}

In many systems, agent performance is naturally measured via time-to-event outcomes. For each task $x_t$ and agent $a_i$, let $T_t(i) > 0$ denote the time-to-completion random variable. We allow right-censoring and introduce a censoring indicator $\delta_t(i) \in \{0,1\}$, where $\delta_t(i)=1$ indicates that the completion time is fully observed.

We define a baseline survival function
\[
S_i(\tau \mid x_t) := \mathbb{P}(T_t(i) > \tau \mid x_t),
\]
which captures heterogeneity across tasks and agents.

To model unobserved task-level difficulty, we introduce a latent frailty variable $\theta_t > 0$, assumed i.i.d. with $\mathbb{E}[\theta_t]=1$. Conditional on $(x_t, \theta_t)$, we assume a proportional frailty model:
\[
\mathbb{P}(T_t(i) > \tau \mid x_t, \theta_t)
= S_i(\tau \mid x_t)^{\theta_t}.
\]

Equivalently, $\theta_t$ acts as a multiplicative scaling of the baseline cumulative hazard. Under this construction, conditional on $(x_t, \theta_t)$, completion times are independent across agents:
$T_t(i) \perp T_t(j) \mid (x_t, \theta_t)$,
and marginalizing over $\theta_t$ induces dependence:
$T_t(i) \not\perp T_t(j) \mid x_t$.  Specific to our setting, one suitable choice for defining the fraility reward function is as follows,
\[
R_t^i
:=
\delta_t(i)\,
S_i(T_t(i)\mid x_t)^{\theta_t}.\]
This survival-based formulation is suitable for modeling reliability, latency, and time-to-success in heterogeneous multi-agent systems.

\subsection{OT-Based Alignment Costs}\label{sec:OT}
Beyond rewards, we also model how well an agent’s capabilities align with the requirements of a task. In many applications, both tasks and agents are naturally described by distributions rather than single feature vectors. For example, an agent may produce a distribution of outcomes (quality levels, response times, error types), while a task may specify a desired target distribution over outcomes.

Let $\mu_i \in \mathcal{P}(\mathcal{Y})$ denote the outcome distribution induced by agent $a_i$ over a measurable space $\mathcal{Y}$, and let $\nu_t \in \mathcal{P}(\mathcal{Y})$ represent the reference or desired outcome distribution associated with task $x_t$. We quantify the mismatch between an agent and a task using the Wasserstein distance $W_c(\nu_t, \mu_i)$ \cite{villani2003topics}, which measures the minimal cost of transporting mass from one distribution to the other under ground cost $c$. This provides a geometrically meaningful notion of alignment that accounts for the full distribution of outcomes rather than just summary statistics.

We note that constructing a fixed reference measure is often non-trivial in practice, particularly when the task distribution is unknown or evolves over time. In the i.i.d.\ setting, a natural choice is to estimate a single representative measure of the environment using a Wasserstein barycenter of the observed task distributions, which provides a principled Fréchet mean in the space of probability measures \cite{Chewi2025}. This yields a stable global reference that summarises the stationary data-generating process. In contrast, in the non-i.i.d.\ setting, where the task distribution may drift or depend on historical interactions, a single global barycenter is generally insufficient to capture temporal heterogeneity. Instead, one may consider time-adaptive or history-dependent barycenters, computed over sliding windows or exponentially weighted empirical measures, thereby producing a sequence of local barycenters that track the evolving environment. This leads to a dynamic reference structure that better reflects non-stationarity while preserving the geometric advantages of the Wasserstein framework.

To incorporate randomness and modeling noise, we define the stochastic alignment cost
\[
W_t^i := W_c(\nu_t, \mu_i) + \epsilon_t^i,
\qquad \epsilon_t^i \sim \mathcal{N}(0, \sigma_i^2).
\]
Its conditional mean is
\[
w^i(x_t) := \mathbb{E}[W_t^i \mid x_t].
\]
This formulation captures both systematic mismatch (via the Wasserstein term) and unpredictable variability (via the noise term $\epsilon_t^i$). Introducing alignment costs at this stage allows us to jointly model (i) how well an agent is suited to a task and (ii) the stochastic rewards that result from performing it, providing a foundational framework for assignment decisions under uncertainty.

\subsection{Orchestration Policy and Objective}

We now formalize the decision-making problem faced by the orchestrator. At each round $t$, after observing the task $x_t$ and past history $\mathcal{H}_t$, the orchestrator selects a randomized policy $\pi_t \in \Delta(\mathcal{A})$, where $\Delta(\mathcal{A})$ denotes the probability simplex over agents. An agent $i_t \sim \pi_t$ is then sampled and assigned to handle task $x_t$. Randomized policies enable exploration and robustness to uncertainty in agent performance.

\paragraph{Per-round expected utility.}
The expected net reward under policy $\pi_t$ is defined as:
\begin{equation}\label{eqn:per_round_utility}
R_t(\pi_t)
:= \mathbb{E}_{i_t \sim \pi_t}\!\left[
r^{i_t}(x_t) - \lambda w^{i_t}(x_t)
\right],
\qquad \lambda > 0,
\end{equation}
where $r^i(x_t)$ and $w^i(x_t)$ denote the conditional expected reward and alignment cost of assigning task $x_t$ to agent $a_i$.
Equivalently,
\[
R_t(\pi_t)
= \pi_t^\top \big(\mathbf{r}(x_t) - \lambda \mathbf{w}(x_t)\big).
\]

\paragraph{Cumulative objective.}
The cumulative expected reward over a finite horizon $T$ is
\[
R(\pi_{1:T}) := \sum_{t=1}^T R_t(\pi_t).
\]

\subsection{Regret and Optimal policy.}

The orchestrator selects policies without access to the true reward $\mathbf{r}(x_t)$ and alignment cost distributions $\mathbf{w}(x_t)$. We evaluate performance against an oracle that has full knowledge of these quantities. This motivates the notion of an optimal policy, which maximizes expected per-round utility under complete information. Since, the oracle is unavailable in practice, the learner deploys a sequence of policies $\{\pi_t\}_{t=1}^T$ based
on partial information. The performance gap between the learned policy and the oracle is quantified via regret, which measures cumulative suboptimality over the horizon. Sublinear regret implies asymptotic convergence to oracle performance. The optimal (oracle) policy at each round is given by:
\begin{equation}\label{eqn:optimal_policy}
\pi_t^*
:= \arg\max_{\pi \in \Delta(\mathcal{A})}
\pi^\top (\mathbf{r}(x_t) - \lambda \mathbf{w}(x_t)).
\end{equation}

\paragraph{Cumulative regret.}
We define the cumulative regret as the performance gap between the optimal policy and the learned policy:
\begin{equation}\label{eqn:cum_regret}
\mathcal{R}_T
:= \sum_{t=1}^T
\big(R_t(\pi_t^*) - R_t(\pi_t)\big).
\end{equation}

\begin{remark}
If tasks are i.i.d.\ and the environment is stationary, the problem reduces to a standard contextual bandit setting. Otherwise, the regret captures additional inefficiency arising from temporal dependence, non-stationary task arrivals, and stochastic alignment costs.
\end{remark}

\section{Theoretical Properties}
\label{sec:theory}
We begin by stating the standing assumptions that remain in force throughout.

\begin{description}\label{des:assumptions}
\item[(A1)] The ground cost $c:\mathcal{Y}\times\mathcal{Y}\to\mathbb{R}_{+}$ is $L$-Lipschitz and bounded.
\item[(A2)] Rewards are uniformly bounded: $0 \le R_t(i) \le R_{\max}$ for all $t$ and all actions $i\in\mathcal{A}$.
\item[(A3)] The frailty variables $(\theta_t)_{t\ge1}$ admit finite exponential moments.
\item[(A4)] Conditional on $(x_t,\theta_t)$, censoring is independent of survival time.
\item[(A5)] The learning rate satisfies $\eta_t = O(t^{-1/2})$.
\end{description}

Assumptions (A1)–(A5) ensure:  
(i) stability of Wasserstein distances under empirical perturbations;  
(ii) sub-exponential behaviour of frailty-adjusted rewards;  
(iii) well-posedness of the induced softmax stochastic approximation dynamics.

We remark that these conditions are standard in the literature on stochastic approximation, Wasserstein-based learning, and frailty-adjusted reward models~\citep{Ambrosio2008,ShalevShwartz2012StochasticApprox,TrillosSlepcev2016GammaConv}. In practice, they are not restrictive: Lipschitz and bounded costs are typical in OT and multi-agent learning applications, bounded rewards naturally arise in reinforcement learning, and learning rates of order $t^{-1/2}$ are widely used to guarantee convergence. Moreover, our framework generalizes several prior works by allowing frailty variables with arbitrary distributions admitting finite exponential moments, rather than restricting to specific parametric forms~\citep{DelBarrio2019Frailty,Wang2020Stochastic}.    

\subsection{Main Results}\label{subsection:main_results}

Following the assumptions in Section \ref{des:assumptions}, we now present the main theoretical guarantees for the OT-regularised bandit model. The regret analysis is conducted for a general exponential-weights (softmax) procedure applied to bounded, OT-regularised reward signals. Importantly, the resulting guarantees depend only on the boundedness of these rewards and are independent of the specific model. The additional modeling components introduced in the setup, namely, correlated rewards, survival-based frailty, and non-i.i.d. task generation serve as a motivating probabilistic framework in which such bounded reward processes naturally arise. They are not directly used in the regret derivation, but instead provide one possible instantiation of the abstract reward model. Proofs are provided in Appendices~\ref{appendix:stability}--\ref{appendix:convergence}.

\begin{theorem}[Properties of OT-Regularized Orchestration]\label{thm:main_ot_bandits}
Assume (A1)--(A5) hold. Let $\mathcal{A}=\{a_1,\dots,a_M\}$ be a finite set of agents, and let $\pi_t \in \Delta(\mathcal{A})$ denote the orchestration policy at round $t$. Define the OT-regularized per-round utility $R_t(\pi_t)$ be as in \eqref{eqn:per_round_utility}
and let the cumulative regret $\mathcal{R}_T$ be as in \eqref{eqn:cum_regret} where the optimal policy $\pi^\ast_t$ is as in \eqref{eqn:optimal_policy}. Let $\phi_t := \pi_t$ denote the policy weight vector. Assume initial conditions are well-defined: $\phi_0 \in \Delta(|\mathcal{A}|-1)$. Then, the following statements hold for $i \in \{1,\cdots,M\}$:
\begin{enumerate}[leftmargin=*]
\item\label{item:sublinera_regret} \textbf{Sublinear OT-Regret.}
Define the policy by $\pi_t(i)
=
\frac{w_t(i)}{\sum_{j=1}^M w_t(j)}$, and
$w_{t+1}(i)
=
w_t(i)\exp\!\Bigl(\eta_t (R_t^i - \lambda W_t^i)\Bigr)$,
where $R_t^i$ and $W_t^i$ are the realized reward and OT costs.
Define the pseudo-regret as $\mathcal{R}_T
:=
\sum_{t=1}^T
\Big(
\max_{i \in \{1,\dots,M\}}
\mathbb{E}[R_t^i - \lambda W_t^i \mid x_t]
-
\mathbb{E}_{i_t \sim \pi_t}[R_t^{i_t} - \lambda W_t^{i_t}]
\Big)$. Then, the cumulative OT-regularized regret satisfies
\[
\mathcal{R}_T = O(\sqrt{T}).
\]

\item\label{item:structure_OT_optimality} \textbf{Structural OT-Optimality.}
For any $i,j \in \mathcal{A}$, if
$\mathbb{E}[R_t^i \mid x_t] = \mathbb{E}[R_t^j \mid x_t]$
and $w^i(x_t) < w^j(x_t)$, then the OT-regularized utilities satisfy
\[
\mathbb{E}[R_t^i \mid x_t] - \lambda w^i(x_t)
>
\mathbb{E}[R_t^j \mid x_t] - \lambda w^j(x_t).
\]

\item\label{item:robustness}\textbf{Margin Robustness under Noisy Alignment.}
Let observed alignment costs be
$\widetilde{W}_t^i = W_t^i + \epsilon_t^i$, with
$\epsilon_t^i \sim \mathcal{N}(0,\sigma^2)$.
Define $\Delta_{ij}(t) := w^j(x_t) - w^i(x_t)$.
If $|\Delta_{ij}(t)| > \sigma \sqrt{2 \log 2}$,
then
\[
\mathbb{P}\Big(
r^i(x_t) - \lambda \widetilde{W}_t^i
<
r^j(x_t) - \lambda \widetilde{W}_t^j
\Big)
< \tfrac{1}{4}.
\]

\item\label{item:convergence_weights} \textbf{Convergence of Orchestration Weights.}
Assume the policy update follows a stochastic approximation scheme:
$\phi_{t+1} - \phi_t
= \eta_t \Big(
\mathrm{Softmax}(\mathbf{r}(x_t) - \lambda \mathbf{w}(x_t)) - \phi_t
\Big)$,
with step sizes $\eta_t$ satisfying $\sum_t \eta_t = \infty$, $\sum_t \eta_t^2 < \infty$. Then $\phi_t$ converges almost surely to an invariant point $\phi_\infty$ of the mean-field ODE
\[
\dot{\phi}
=
\mathrm{Softmax}\big(\mathbb{E}[\mathbf{r}(x)] - \lambda \mathbb{E}[\mathbf{w}(x)]\big)
- \phi.
\]

\item\label{item:consistency_weights} \textbf{Uniform Consistency of Empirical Rewards.}
Let $\hat r_t(i) := \frac{1}{t}\sum_{s=1}^t R_s^i$
be the empirical average reward and let \(|R_t^i|\le R_{\max}\) be the bounded rewards. Then, we have
\[
\sup_{i \in \mathcal{A}}
\left|
\hat r_t(i)
- \frac{1}{t}\sum_{s=1}^t \mathbb{E}[R_s^i \mid \mathcal{H}_{s-1}]
\right|
\xrightarrow[t\to\infty]{\mathrm{a.s.}} 0.
\]
In particular, in the i.i.d. case this implies
\[
\sup_{i \in \mathcal{A}}
\left|
\hat r_t(i) - \mathbb{E}[R_t^i]
\right|
\xrightarrow[t\to\infty]{\mathrm{a.s.}} 0.
\]
\end{enumerate}
\end{theorem}
We note that Theorem \ref{thm:main_ot_bandits} provides a foundational characterization of OT-regularized bandit learning. It integrates distributional alignment into the reward structure, yielding principled agent differentiation~\ref{item:structure_OT_optimality}, sublinear cumulative regret~\ref{item:sublinera_regret}, robustness to noisy observations~\ref{item:robustness}, convergence of orchestration weights~\ref{item:convergence_weights}, and uniform consistency of empirical rewards~\ref{item:consistency_weights}. The framework enables heterogeneous agents to learn coordinated policies under distributionally-aware uncertainty, supporting robust, adaptive orchestration in stochastic, partially observable environments. We refer the reader to Appendix \ref{appendix:proofs} for a proof of Theorem \ref{thm:main_ot_bandits}.

\section{The proposed algorithm}
BOT-Orch combines bandit-based selection with OT alignment and survival-based rewards to handle heterogeneous, non-stationary agents. In the \emph{i.i.d. task setting} (Algorithm~\ref{alg:bot-iid} in the Appendix), a Boltzmann policy selects agents using exponentially smoothed rewards penalized by OT misalignment, balancing exploitation and alignment, while survival rewards capture latent difficulty, censoring, and reliability. The \emph{non-i.i.d. extension} (Algorithm~\ref{alg:bot-noniid} in the Appendix) allows history-dependent task distributions and reward updates, handling temporal correlations, non-stationarity, and regime shifts via a correction term that encodes memory effects. Overall, BOT-Orch is a \emph{risk-aware, alignment-regularized bandit algorithm} in distributional space, where OT enforces task-agent compatibility and survival rewards provide robustness to censoring and heterogeneity. Its modular design accommodates alternative OT solvers, survival models, and exploration schemes, and motivates regret analysis under composite reward--cost objectives as well as questions of convergence and adaptivity in non-stationary environments.

\section{Synthetic Experiments}
\label{suntehtic_experiments}

%
%
\subsection{Dataset and Task Description}
We consider both i.i.d.\ and non-i.i.d.\ regimes (see Fig.~\ref{fig:iid}--\ref{fig:non_iid} in the Appendix) in order to evaluate BOT-Orch under both stationary and evolving environments. The i.i.d.\ settings isolate performance under stable task distributions, whereas the non-i.i.d.\ settings introduce different forms of temporal dependence and distributional shift.

In the stationary setting, we consider two environments. In \textbf{IID-G}, tasks $x_t$ are sampled i.i.d.\ over $\mathcal{X}$ and agent rewards are drawn from fixed Gaussian distributions with matched means ($\approx 0.5$) but heterogeneous higher-order structure. This setting evaluates whether the orchestration policy can distinguish between agents beyond mean reward alone. In \textbf{IID-M}, tasks are sampled i.i.d.\ from a half-moons distribution, while rewards are drawn from fixed distributions with similar means but differing variance, skewness, and bimodality, thereby introducing additional distributional heterogeneity.

We then consider three forms of non-stationarity. In \textbf{NonIID-PS}, rewards are piecewise-stationary, with fixed means but variance shifts occurring at unknown changepoints, inducing distributional shift without changing expected reward. In \textbf{NonIID-SD}, non-stationarity is introduced gradually through sinusoidal drift in the reward means over time. Finally, in \textbf{NonIID-BB}, latent reward means evolve according to a temporally correlated Brownian-bridge process with fixed endpoints, capturing smoothly varying latent dynamics.

All experiments were conducted on a standard x86\_64 CPU platform using 2 CPU cores, 13.6\,GB RAM, and 107\,GB disk storage.
%
%

\subsection{Baselines}
\label{baselines}
We compare BOT-Orch against the following baselines. All baselines operate on the same task stream and are evaluated with identical metrics and random seeds.
\begin{enumerate}[label={$\bullet$},leftmargin=*,itemsep=1pt,topsep=1pt]
  \item \textbf{BOT-Orch (ours).}
  Uses alignment-adjusted rewards $r_t(i) = \hat r_t(i) - \lambda W_t(i)$ and samples the selected agent $i_t$ from a softmax policy over $r_t(i)$ (Algorithm~\ref{algorithm1}--\ref{algorithm2}). This couples OT alignment with sequential exploration/exploitation.
  \item \textbf{No-OT ($\lambda = 0$).}
  Ablation removing the OT alignment term. The policy is computed from $\hat r_t(i)$ only. This isolates the contribution of distributional alignment.
  \item \textbf{Random.}
  Uniformly samples an agent each round. This serves as a naive lower bound that does not learn from observations.
  \item \textbf{UCB1 (MAB).}
  A classical multi-armed bandit baseline that selects the agent with the highest upper confidence bound based on empirical reward estimates, balancing exploration and exploitation without OT alignment.
\end{enumerate}
We exclude baselines that require \emph{full-information} feedback, i.e., access to the rewards of all agents at every round, since our setting assumes standard \emph{bandit feedback}, where only the reward of the selected agent is observed. This includes simple greedy strategies based on exponentially smoothed reward estimates, which are not directly applicable under partial feedback without additional exploration or uncertainty-estimation mechanisms.

%

\subsection{Evaluation Metrics}
We use the following metrics for experimental comparison, computed per seed and aggregated as mean$\pm$95\% confidence intervals across seeds.
\label{evaluation_metrics}
\begin{enumerate}[label={$\bullet$},leftmargin=*,itemsep=1pt,topsep=1pt]
  \item \textbf{Cumulative net utility.}
  We evaluate the OT-regularized net utility $U_t(i_t) \;=\; R_t(i_t) - \lambda\, W_t(i_t)$, and report the cumulative net utility $\sum_{t=1}^T U_t(i_t)$. This captures the overall performance when explicitly trading off reward quality and alignment cost via $\lambda$.

  \item \textbf{Cumulative alignment cost.}
  We report $\sum_{t=1}^T W_t(i_t)$, the cumulative OT alignment cost incurred by the selected agent assignments. Lower values indicate more distributionally aligned agent-task matching.

  \item \textbf{Oracle regret.}
  We report OT-regularized cumulative regret relative to the best alignment-adjusted agent at each round: $
  \sum_{t=1}^T (\max_{i} U_t(i) - U_t(i_t))$ measuring efficiency loss due to suboptimal selection under uncertainty and non-stationarity.
\end{enumerate}
%
%
%
\subsection{Results}

BOT-Orch achieves the highest cumulative net utility and lowest oracle regret across all environments (Table \ref{tab:align_cost_ci_4col}-\ref{tab:net_utility_ci_4col}), consistently outperforming No-OT, Random, and UCB1. The improvement is particularly pronounced in non-i.i.d.\ settings, where BOT-Orch maintains low regret under distributional shifts such as piecewise variance changes and smooth drift. These results indicate that incorporating OT-based alignment enables more effective adaptation to heterogeneous and non-stationary agent behaviour.

%
In Appendix~\ref{add_suntehtic_experiments}, we evaluate BOT-Orch on additional synthetic benchmarks using survival-based metrics, showing that BOT-Orch achieves best performance. We further conduct an ablation study on the parameter $\lambda$ (Appendix \ref{sec:ablation_synthetic}). As $\lambda$ increases, BOT-Orch exhibits performance improvements, with large gains relative to $\lambda=0$ and diminishing returns at higher values.
%
%
\begin{table*}[t]
\centering
\scriptsize
\setlength{\tabcolsep}{4pt}
\begin{tabular}{|l|cccc|}
\hline
Environment & BOT-Orch & No-OT ($\lambda$=0) & Random & UCB1 (MAB) \\
\hline
\textbf{IID} & \multicolumn{4}{c|}{} \\
IID-G & \textbf{537.43$\pm$16.57} & 656.31$\pm$22.64 & 673.99$\pm$19.08 & 664.89$\pm$29.85 \\
IID-M & \textbf{459.84$\pm$9.27} & 545.85$\pm$17.74 & 553.07$\pm$19.16 & 549.48$\pm$7.40 \\
\hline
\textbf{Non-IID} & \multicolumn{4}{c|}{} \\
NonIID-BB & \textbf{410.04$\pm$71.06} & 503.14$\pm$91.24 & 520.28$\pm$101.61 & 496.60$\pm$79.57 \\
NonIID-PS & \textbf{571.43$\pm$18.71} & 713.81$\pm$25.79 & 734.65$\pm$30.08 & 729.51$\pm$28.34 \\
NonIID-SD & \textbf{564.13$\pm$18.11} & 717.24$\pm$22.51 & 707.69$\pm$18.33 & 708.60$\pm$16.06 \\
\hline
\end{tabular}
\caption{Cumulative Alignment Cost (mean \(\pm\) 95\% CI across 5 seeds; \(T=200\)).  IID condition uses Algorithm~1; Non-IID uses Algorithm~2. Best in bold.}
\label{tab:align_cost_ci_4col}
\end{table*}

\begin{table*}[t]
\centering
\scriptsize
\setlength{\tabcolsep}{1pt}
\makebox[\textwidth][c]{%
\begin{tabular}{|l|cccc|cccc|}
\cline{1-9} 
& \multicolumn{4}{c|}{\textbf{Cumulative Net Utility}} & \multicolumn{4}{c|}{\textbf{Oracle Regret}} \\
\hline
\multicolumn{1}{|l|}{Environment} & BOT-Orch & No-OT ($\lambda$=0) & Random & UCB1 (MAB) & BOT-Orch & No-OT ($\lambda$=0) & Random & UCB1 (MAB) \\
\hline
\multicolumn{1}{|l|}{\textbf{IID}} & \multicolumn{4}{c|}{} & \multicolumn{4}{c|}{} \\
\multicolumn{1}{|l|}{IID-G} & \textbf{-467.528$\pm$17.60} & -588.34$\pm$23.19 & -605.71$\pm$18.91 & -603.31$\pm$29.13 & \textbf{122.65$\pm$5.68} & 243.47$\pm$14.26 & 260.83$\pm$11.59 & 258.43$\pm$18.11 \\
\multicolumn{1}{|l|}{IID-M} & \textbf{-386.10$\pm$10.04} & -478.54$\pm$20.49 & -482.81$\pm$18.65 & -484.57$\pm$9.35 & \textbf{74.70$\pm$5.45} & 166.25$\pm$11.99 & 170.52$\pm$10.99 & 172.28$\pm$8.40 \\
\hline
\multicolumn{1}{|l|}{\textbf{Non-IID}} & \multicolumn{4}{c|}{} & \multicolumn{4}{c|}{} \\
\multicolumn{1}{|l|}{NonIID-BB} & \textbf{-335.25$\pm$70.42} & -434.01$\pm$86.69 & -449.34$\pm$100.73 & -426.09$\pm$79.48 & \textbf{76.85$\pm$7.42} & 175.61$\pm$16.44 & 190.94$\pm$31.44 & 167.69$\pm$19.20 \\
\multicolumn{1}{|l|}{NonIID-PS} & \textbf{-498.28$\pm$22.49} & -642.77$\pm$25.89 & -666.10$\pm$29.31 & -663.22$\pm$27.77 & \textbf{126.77$\pm$11.68} & 271.26$\pm$11.84 & 294.58$\pm$20.25 & 291.71$\pm$14.45 \\
\multicolumn{1}{|l|}{NonIID-SD} & \textbf{-492.92$\pm$18.74} & -644.68$\pm$21.62 & -637.82$\pm$18.54 & -643.69$\pm$17.20 & \textbf{128.77$\pm$13.30} & 280.53$\pm$10.81 & 273.67$\pm$10.74 & 279.54$\pm$11.52 \\
\hline
\end{tabular}}
\caption{Cumulative Net Utility and Oracle Regret (mean \(\pm\) 95\% CI across 5 seeds; \(T=200\)). IID condition uses Algorithm~1; Non-IID uses Algorithm~2. Best in bold.}
\label{tab:net_utility_ci_4col}
\end{table*}

\begin{table*}[t]
\centering
\scriptsize
\setlength{\tabcolsep}{3pt}
\makebox[\textwidth][c]{
\begin{tabular}{|l|cccc|cccc|}
\hline
& \multicolumn{4}{c|}{\textbf{Cumulative Net Utility}}
& \multicolumn{4}{c|}{\textbf{Oracle Regret}} \\
\hline
Environment & BOT-Orch & No-OT ($\lambda{=}0$) & Random & UCB1
 & BOT-Orch & No-OT ($\lambda{=}0$) & Random & UCB1 \\
\hline
\textbf{IID }
  & \textbf{108.84{$\pm$2.22}}
  & 103.37{$\pm$2.51}
  & 83.98{$\pm$5.92}
  & 80.55{$\pm$4.88}
  & \textbf{2.29{$\pm$2.11}}
  & 9.80{$\pm$1.61}
  & 28.72{$\pm$6.03}
  & 31.31{$\pm$3.98} \\
\textbf{Non-IID}
  & \textbf{110.61{$\pm$1.03}}
  & 103.17{$\pm$1.80}
  & 79.78{$\pm$5.52}
  & 85.82{$\pm$4.91}
  & \textbf{0.59{$\pm$0.93}}
  & 10.14{$\pm$1.17}
  & 32.97{$\pm$4.87}
  & 26.26{$\pm$4.28} \\
\hline
\multicolumn{9}{c}{} \\ 
\multicolumn{9}{c}{} \\ 
\hline
& \multicolumn{4}{c|}{\textbf{Team Accuracy}}
& \multicolumn{4}{c|}{\textbf{Cumulative Alignment Cost}} \\
\hline
Environment & BOT-Orch & No-OT ($\lambda{=}0$) & Random & UCB1
& BOT-Orch & No-OT ($\lambda{=}0$) & Random & UCB1 \\
\hline
\textbf{IID }
  & \textbf{0.980{$\pm$0.010}}
  & 0.907{$\pm$0.022}
  & 0.917{$\pm$0.020}
  & 0.902{$\pm$0.026}
  & \textbf{1.28{$\pm$0.53}}
  & 10.51{$\pm$1.61}
  & 10.28{$\pm$2.01}
  & 11.14{$\pm$1.33} \\
\textbf{Non-IID}
  & \textbf{0.993{$\pm$0.007}}
  & 0.905{$\pm$0.016}
  & 0.905{$\pm$0.024}
  & 0.919{$\pm$0.025}
  & \textbf{0.85{$\pm$0.23}}
  & 10.84{$\pm$1.18}
  & 11.69{$\pm$1.62}
  & 9.46{$\pm$1.43} \\
\hline
\end{tabular}
}
\caption{%
  Deployment metrics across all four methods and both experimental
  conditions (Mean $\pm$ 95\% CI across 30 seeds, $T{=}114$,
  $\lambda{=}3.0$).
  IID condition uses Algorithm~1; Non-IID uses Algorithm~2
  (ID patients rounds 1--57, shifted rounds 58--114). Best in bold.\vspace{-5pt}}
\label{tab:combined_results}
\end{table*}

\section{Semi-Synthetic Experiments: Human-AI Triage Under\\Deployment Shift}
\label{sec:real_world_experiments}
%
%
\subsection{Dataset and Task Description}
We use the Breast Cancer Wisconsin (Diagnostic) dataset, consisting of 569 patient cases with 30 numerical features and a binary target (malignant vs.\ benign). The data is split into 60\% train, 20\% calibration, 10\% Test-ID, and 10\% Test-Shift. To simulate deployment shift, we perturb a subset of features in Test-Shift with additive Gaussian noise $\mathcal{N}(0,0.64)$ and a $+0.5$ standard-unit bias, modelling a change in patient population at deployment. At each round $t$, a patient $x_t$ arrives and the selected agent receives a reward of 1 if it classifies the patient correctly and 0 otherwise, under standard bandit feedback.

We consider two agents: an AI classifier and a human proxy with complementary accuracy, where the human performs better on shifted cases while the AI performs better in-distribution. Full model details and accuracy statistics are provided in Appendix~\ref{sec:app_setting}. All experiments were conducted on a standard x86\_64 CPU platform using 2 CPU cores, 13.6\,GB RAM, and 107\,GB disk storage.

\begin{figure}[t]
\centering
\includegraphics[width=1.0\linewidth]{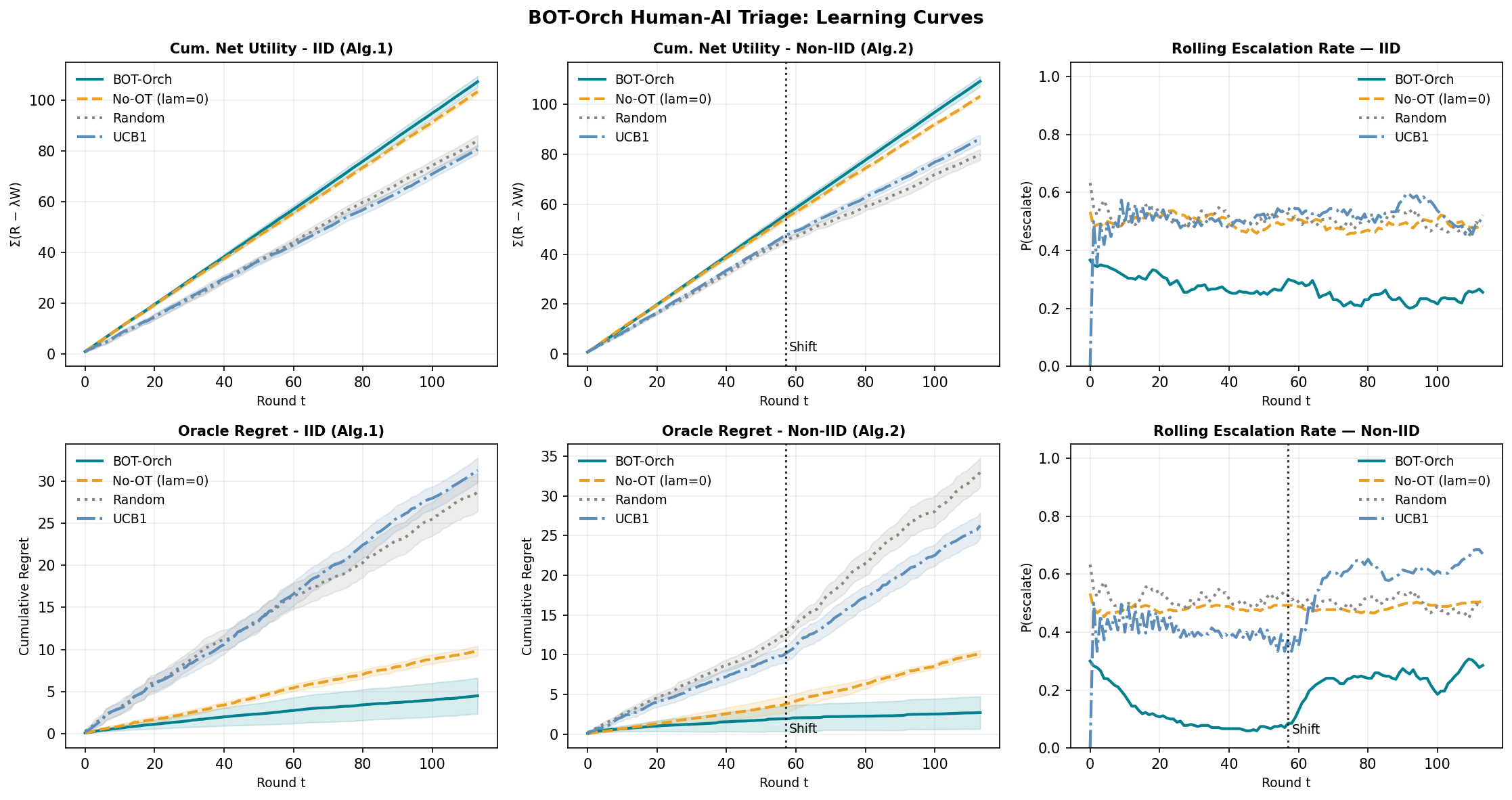}
\caption{%
  \textbf{Cumulative learning curves.}
  \textit{Top row}: cumulative net utility.
  \textit{Bottom row}: oracle regret.
  \textit{Right column}: rolling escalation rate
  (window $w{=}8$ rounds).
  Left panels: IID condition (Algorithm~1);
  middle panels: Non-IID condition (Algorithm~2),
  with the dotted vertical line marking the shift onset
  at round~57;
  right panels: escalation rate evolution.
  }
\label{fig:learning_curves}
\end{figure}

\subsection{Baselines}

We compare the same four methods used in the synthetic experiments, now applied to the clinical triage task. We use $\alpha=0.90$, $\eta=5.0$, $\lambda=3.0$, and, for the non-i.i.d.\ variant, the same history correction form as Algorithm~2 with $\beta=0.05$. Each episode has $T=114$ rounds, and we average over $n=30$ random seeds (patient orderings). We report mean $\pm$ 95\% CI across seeds.

\subsection{Evaluation Metrics}

From Section~\ref{evaluation_metrics}, we use \textbf{cumulative net utility}, \textbf{cumulative alignment cost}, and \textbf{oracle regret}, defined as above. We additionally report metrics specific to the human--AI deferral setting, namely \textbf{team accuracy}, defined as the fraction of patients correctly classified by the selected agent (AI or human), and \textbf{escalation rate}, defined as the fraction of patients routed to the human expert.

These additional metrics capture complementary aspects of orchestration quality beyond cumulative utility and regret. Together, these metrics provide insight into the trade-off between automation and human intervention, particularly under deployment shift and uncertainty.
%
%
%
\subsection{Results}

Table~\ref{tab:combined_results} summarizes the results. BOT-Orch achieves the highest cumulative net utility and lowest oracle regret in both regimes, outperforming all baselines. In the IID condition, it improves net utility by $+5.5$ over No-OT while reducing alignment cost by an order of magnitude. In the Non-IID setting, the gain is larger, reaching $110.61 \pm 1.03$ net utility and $0.59 \pm 0.93$ regret compared to $103.17 \pm 1.80$ and $10.14 \pm 1.18$ for No-OT.

The No-OT ablation confirms that the OT term drives the improvement: removing it increases regret and alignment cost. BOT-Orch also exhibits targeted escalation under shift, routing fewer patients overall while increasing escalation on shifted cases, indicating effective adaptation to distributional mismatch. Figure~\ref{fig:learning_curves} further shows that BOT-Orch separates from all baselines early and maintains the performance gap over time. In the Non-IID setting, the escalation rate increases after the shift, demonstrating online adaptation to the changing population. Furthermore, the rolling escalation rate in Figure~\ref{fig:learning_curves} shows how the policy adapts over time, increasing reliance on the human expert when distributional shift or uncertainty rises.

Additional results are presented in Appendix \ref{additional_results_Med}, including more results on the escalation rate, diagnostic analysis, and an ablation study on the parameter $\lambda$. We observe that, as $\lambda$ increases, performance improves markedly peaking around $\lambda \approx 3.0$, after which gains saturate or slightly diminish at very high values.

\section{Conclusion}

We introduced BOT-Orch, a framework for uncertainty-aware agent orchestration that formulates delegation as an OT-regularized bandit problem over heterogeneous agents. By combining sequential decision-making with distributional alignment, BOT-Orch jointly accounts for agent reliability, alignment quality, and uncertainty when selecting agents under stochastic and potentially non-stationary environments. We established theoretical guarantees including sublinear OT-regularized regret, robustness to noisy alignment, and convergence of orchestration weights. Empirically, BOT-Orch consistently outperformed standard bandit and heuristic baselines across synthetic and semi-synthetic settings, particularly under distribution shift and heterogeneous agent behaviour.

More broadly, our results suggest that incorporating distributional structure directly into orchestration policies can substantially improve robustness and adaptability in modern multi-agent AI systems. This is especially relevant for emerging ecosystems composed of specialised models, reasoning systems, and tool-using agents, where orchestration increasingly plays a central role in overall system performance.

Several limitations remain. First, computing Wasserstein distances can become computationally expensive in high-dimensional settings, potentially limiting scalability in real-time applications. Second, the alignment parameter $\lambda$ introduces a reward--alignment trade-off that may require calibration across environments. Third, the framework assumes access to suitable task-specific reference distributions, which may be difficult to obtain or estimate reliably in open-world settings.

Future work will investigate more scalable alignment mechanisms, adaptive calibration strategies, and evaluation in larger-scale real-world orchestration environments. We also believe that extending uncertainty-aware orchestration beyond single-agent delegation toward richer coordination and interaction settings is a promising direction for further study.
\bibliographystyle{apalike}
\bibliography{bibliography}

\newpage

\appendix

\begin{center}

{\Large\bf SUPPLEMENTARY MATERIALS}

\end{center}

\setcounter{equation}{0}

\section*{Contents}

\startcontents[sections]
\printcontents[sections]{l}{1}{\setcounter{tocdepth}{2}}

\newpage
\section{Algorithm: Bandit--OT Orchestration with Survival-Based Rewards}

\subsection{i.i.d. Task Arrivals}
\begin{algorithm}[h!]
   \caption{BOT-Orch: i.i.d. Task Version}
   \label{alg:bot-iid}
\begin{algorithmic}[1]
   \State \textbf{Input:} Agents $\mathcal{A} = \{a_1,\dots,a_M\}$, OT cost $c$, alignment weight $\lambda \ge 0$, horizon $T$, survival model $S_i(\cdot)$, learning rate $\alpha$, inverse temperature $\eta_t$
   \State Initialize estimated rewards $\hat{r}_0(i) = 0$ and frailty $\theta_0 = 1$ for all $i \in \mathcal{A}$
   \For{$t = 1$ \textbf{to} $T$}
       \State Sample task $x_t \sim \mathcal{P}_X$ \Comment{i.i.d. task sampling}
       \State Sample latent frailty $\theta_t \sim p(\theta)$
       \For{$i = 1$ \textbf{to} $M$}
           \State Compute agent output distribution $\mu_i$
           \State Compute OT alignment cost $\mathcal{W}_t(i) = W_c(\nu_t, \mu_i) + \epsilon_t(i)$
           \State Sample survival time $T_t(i)$ and censoring $\delta_t(i)$
           \State Compute reward: $R_t(i) = \delta_t(i)\, S_i(T_t(i) \mid x_t)^{\theta_t}$
           \State Update estimated reward: $\hat{r}_t(i) = \alpha\hat{r}_{t-1}(i) + (1-\alpha)R_t(i)$
       \EndFor
       \State Compute orchestration policy:
       \[
           \pi_t(i) = \frac{\exp\!\left(\eta_t[\hat{r}_t(i)-\lambda \mathcal{W}_t(i)]\right)}
           {\sum_{j=1}^{M} \exp\!\left(\eta_t[\hat{r}_t(j)-\lambda \mathcal{W}_t(j)]\right)}
       \]
       \State Sample and execute agent $i_t \sim \pi_t$
   \EndFor
   \State \textbf{Output:} Policies $\{\pi_t\}_{t=1}^T$ and cumulative reward
\end{algorithmic}
\label{algorithm1}
\end{algorithm}
\newpage
\subsection{Non-i.i.d. Task Arrivals}
\begin{algorithm}[h!]
   \caption{BOT-Orch: Non-i.i.d. Task Version}
   \label{alg:bot-noniid}
\begin{algorithmic}[1]
   \State \textbf{Input:} Same as Algorithm~\ref{alg:bot-iid}
   \State Initialize $\hat{r}_0(i) = 0$ and $\theta_0 = 1$ for all $i \in \mathcal{A}$
   \For{$t = 1$ \textbf{to} $T$}
       \State Sample task $x_t \sim \mathcal{P}(x_t \mid \mathcal{H}_t)$ \Comment{history-dependent, non-i.i.d.}
       \State Sample latent frailty $\theta_t \sim p(\theta_t \mid \mathcal{H}_t)$
       \For{$i = 1$ \textbf{to} $M$}
           \State Compute agent output distribution $\mu_i$
           \State Compute OT alignment cost $\mathcal{W}_t(i) = W_c(\nu_t, \mu_i) + \epsilon_t(i)$
           \State Sample survival time $T_t(i)$ and censoring $\delta_t(i)$
           \State Compute reward: $R_t(i) = \delta_t(i)\, S_i(T_t(i) \mid x_t)^{\theta_t}$
           \State Update estimated reward using history: \Comment{temporal dependence from past rewards}
           \[
               \hat{r}_t(i) = \alpha\hat{r}_{t-1}(i) + (1-\alpha)R_t(i) + f_i(\mathbf{R}_{1:t-1})
           \]
       \EndFor
       \State Compute orchestration policy:
       \[
           \pi_t(i) = \frac{\exp\!\left(\eta_t[\hat{r}_t(i)-\lambda \mathcal{W}_t(i)]\right)}
           {\sum_{j=1}^{M} \exp\!\left(\eta_t[\hat{r}_t(j)-\lambda \mathcal{W}_t(j)]\right)}
       \]
       \State Sample and execute agent $i_t \sim \pi_t$
       \State Update history $\mathcal{H}_{t+1} = \mathcal{H}_t \cup \{(x_t, \mathbf{R}_t, \mathcal{W}_t)\}$
   \EndFor
   \State \textbf{Output:} Policies $\{\pi_t\}_{t=1}^T$ and cumulative reward
\end{algorithmic}
\label{algorithm2}
\end{algorithm}
\paragraph{Derivation of Algorithms from the Theoretical Framework (Special Case Realisation).}
The BOT-Orch algorithms in Algorithms~\ref{alg:bot-iid} and \ref{alg:bot-noniid} can be formally interpreted as special cases of the abstract framework introduced in Theorem~\ref{thm:main_ot_bandits}. At the theoretical level, the model is defined in terms of an unobserved alignment-adjusted reward process
\[
r_t(i) = \mathbb{E}[R_t(i)\mid x_t] - \lambda W_c(\mu_i,\nu_t),
\]
over which the learning dynamics are characterised via exponential-weights updates on the simplex. The algorithms instantiate this framework by specifying an explicit stochastic realisation of the reward process together with a consistent estimator of its conditional expectation. In particular, the survival–frailty construction generates bounded random variables \(R_t(i)\) whose conditional expectation coincides with the abstract reward functional assumed in the theory, thereby embedding the model within a well-defined stochastic process satisfying assumptions (A2)–(A4). 

The empirical quantity \(\hat r_t(i)\), defined via exponential smoothing in the i.i.d. case and augmented with a history-dependent correction term in the non-i.i.d. case, constitutes a Robbins–Monro stochastic approximation of \(\mathbb{E}[R_t(i)\mid x_t]\) \cite{RobbinsMonro1951}, ensuring asymptotic consistency under the respective dependence structures. Substituting this estimator into the theoretical objective yields a computable approximation of \(r_t(i)\), while the entropy-regularised optimisation over \(\Delta^{|\mathcal{A}|-1}\) induces the softmax policy used in the algorithm. 

Consequently, the i.i.d. algorithm corresponds to the stationary special case in which \((x_t,\theta_t)\) are independent draws and the induced reward process is temporally homogeneous, whereas the non-i.i.d. algorithm generalises this construction to an adapted filtration \(\mathcal{H}_t\), allowing for history-dependent task and frailty evolution while preserving boundedness and measurability of the reward sequence. In both cases, the algorithm complements by providing an explicit implementation of the main theorem, showing that BOT-Orch is a realised instance of the general OT-regularised exponential-weights framework under different assumptions on the data-generating process.

\section{Appendix: Missing Proofs}
\label{appendix:proofs}

This appendix contains full proofs of the statements in Section~\ref{sec:theory}. 




\subsection{Proof of ~\ref{item:sublinera_regret} in Theorem \ref{thm:main_ot_bandits} (Sublinear OT-Regret)}
\label{appendix:regret}

\begin{theorem}[Sublinear OT-Regret]\label{thm:regret}
Assume \emph{(A1)--(A5)}. Let the BOT--Orch policy be defined by
\[
\pi_t(i)
=
\frac{w_t(i)}{\sum_{j=1}^M w_t(j)},
\qquad
w_{t+1}(i)
=
w_t(i)\exp\!\Bigl(\eta_t (R_t^i - \lambda W_t^i)\Bigr),
\]
where $R_t^i$ and $W_t^i$ are the realized reward and OT costs.
Define the pseudo-regret
\[
\mathcal{R}_T
:=
\sum_{t=1}^T
\Big(
\max_{i \in \{1,\dots,M\}}
\mathbb{E}[R_t^i - \lambda W_t^i \mid x_t]
-
\mathbb{E}_{i_t \sim \pi_t}[R_t^{i_t} - \lambda W_t^{i_t}]
\Big).
\]

If the inverse-temperature schedule satisfies $\eta_t \asymp t^{-1/2}$ by Assumption (A5),
then
\[
\mathcal{R}_T = O(\sqrt{T}).
\]
\end{theorem}
\begin{proof}
The expected per-round utility of a policy $\pi_t \in \Delta(\mathcal A)$ is
\[
R_t(\pi_t)
=
\sum_{i=1}^M \pi_t(i)\bigl(r^i(x_t)-\lambda w^i(x_t)\bigr).
\]

For each agent $i$, define the realized instantaneous utility
\[
u_t^i := R_t^i - \lambda W_t^i,
\]
and its conditional expectation
\[
\bar u_t^i := r^i(x_t) - \lambda w^i(x_t).
\]

We compare against a fixed action $i^\ast \in \{1,\dots,M\}$. The oracle policy selects
\[
R_t(\pi_t^\ast)
=
\max_{1\le i\le M}
\bigl(r^i(x_t)-\lambda w^i(x_t)\bigr),
\]
i.e. a point mass on the best action in each round.

Define the log-partition potential
\[
\Phi_t := \log\sum_{i=1}^M w_t(i).
\]

By assumption (A1), there exist constants $a<b$ such that for all $t,i$,
\[
u_t^i \in [a,b].
\]

\paragraph{Step 1: Constant inverse temperature.}
Assume $\eta_t \equiv \eta >0$. Then
\[
\Phi_{t+1}-\Phi_t
=
\log\sum_{i=1}^M
\pi_t(i)\exp\!\bigl(\eta u_t^i\bigr).
\]

Applying Hoeffding's lemma,
\[
\Phi_{t+1}-\Phi_t
\le
\eta \sum_{i=1}^M \pi_t(i) u_t^i
+
\frac{\eta^2}{8}(b-a)^2.
\]

Summing over $t=1,\dots,T$ gives
\[
\Phi_{T+1}-\Phi_1
\le
\eta \sum_{t=1}^T \sum_{i=1}^M \pi_t(i) u_t^i
+
\frac{\eta^2 T}{8}(b-a)^2.
\]

On the other hand, for any fixed $i^\ast$,
\[
\Phi_{T+1}
\ge
\log w_{T+1}(i^\ast)
=
\log w_1(i^\ast)
+
\eta \sum_{t=1}^T u_t^{i^\ast}.
\]

With uniform initialization $w_1(i)=1$, we obtain
\[
\sum_{t=1}^T u_t^{i^\ast}
-
\sum_{t=1}^T \sum_{i=1}^M \pi_t(i) u_t^i
\le
\frac{\log M}{\eta}
+
\frac{\eta T}{8}(b-a)^2.
\]

Choosing
\[
\eta
=
\sqrt{\frac{8\log M}{T(b-a)^2}}
\]
yields
\[
\mathcal R_T
\le
(b-a)\sqrt{\frac{T\log M}{2}}.
\]

Taking conditional expectations with respect to $\mathcal H_{t-1}$,
\[
\mathbb E[u_t^i \mid x_t,\mathcal H_{t-1}]
=
r^i(x_t)-\lambda w^i(x_t),
\]
gives the same regret bound in expectation.

\paragraph{Step 2: Time-varying inverse temperature.}
Let $\eta_t \asymp t^{-1/2}$. Then
\[
\Phi_{t+1}-\Phi_t
\le
\eta_t \sum_{i=1}^M \pi_t(i) u_t^i
+
\frac{\eta_t^2}{8}(b-a)^2.
\]

Summing over $t$ yields
\[
\mathcal R_T
\le
\frac{\log M}{\eta_T}
+
\frac{(b-a)^2}{8}\sum_{t=1}^T \eta_t^2.
\]

Since $\eta_t = c t^{-1/2}$,
\[
\frac{1}{\eta_T}=O(\sqrt T),
\qquad
\sum_{t=1}^T \eta_t^2 = O(\log T),
\]
so
\[
\mathcal R_T = O(\sqrt T).
\]

If $u_t^i$ admits a martingale decomposition with bounded increments, Azuma--Hoeffding implies that deviations between realized and conditional cumulative utilities are $O(\sqrt T)$ (up to log factors), which does not change the overall regret rate.
\end{proof}
\subsection{Proof of \ref{item:structure_OT_optimality} in Theorem \ref{thm:main_ot_bandits} (Lipschitz stability of alignment-adjusted rewards)}
\label{appendix:stability}

\begin{proposition}[OT geometry and preference stability]
\label{prop:ot_geometry_preference}
Assume (A1). Let $\mathcal{A} = \{a_1,\dots,a_M\}$ be the set of agents, $\mathcal{X} \subset \mathbb{R}^d$ be the measurable space and $\mathcal{Y}$ denote the outcome space. For $i\in \{1,\cdots,M\}$ let each agent be denoted as $a_i$ and let $\mu_i \in \mathcal{P}(\mathcal{Y})$ be its source outcome distribution. For each task $x_t \in \mathcal{X}$, let $\nu(x_t) \in \mathcal{P}(\mathcal{Y})$ denote the task-induced target outcome distribution. Define the cost as $W_t^i := W_c(\nu_t,\mu_i) + \epsilon_t^i$, such that
$\mathbb{E}[\epsilon_t^i \mid x_t] = 0$. Let the conditional cost $w^i(x_t) := \mathbb{E}[W_t^i \mid x_t] = W_c(\nu_t,\mu_i)$ and conditional reward as $r^i(x_t) := \mathbb{E}[R_t^i \mid x_t]$. Then, the following statements hold:
\begin{enumerate}
\item[\textnormal{(i)}] For any $\nu_t,\nu_t' \in \mathcal{P}(\mathcal{Y})$, we have
\[
\big| W_c(\nu_t,\mu_i) - W_c(\nu_t',\mu_i) \big|
\le L\, W_1(\nu_t,\nu_t').
\]

\item[\textnormal{(ii)}] If $r^i(x_t)=r^j(x_t)$, then for any $\lambda>0$, we have
\[
W_c(\nu_t,\mu_i) < W_c(\nu_t,\mu_j)
\;\Longleftrightarrow\;
r^i(x_t) - \lambda w^i(x_t) > r^j(x_t) - \lambda w^j(x_t).
\]
\end{enumerate}
\end{proposition}
\begin{proof} We outline the proof in a sequence.

\textbf{(i)} By definition of $w^i(x_t;\nu_t)$ and $w^i(x_t;\nu_t')$,
\begin{align}
\big|w^i(x_t;\nu_t)-w^i(x_t;\nu_t')\big|
&= \big|W_c(\nu_t,\mu_i)-W_c(\nu_t',\mu_i)\big|.
\end{align}
Since $c$ is $L$-Lipschitz by (A1), standard Wasserstein stability implies
\begin{align}
\big|W_c(\nu_t,\mu_i)-W_c(\nu_t',\mu_i)\big|
\le L\,W_1(\nu_t,\nu_t'),
\end{align}
which proves (i).

\textbf{(ii)} Using $r^i(x_t)=r^j(x_t)$,
\begin{align}
&\big(r^i(x_t)-\lambda w^i(x_t)\big)-\big(r^j(x_t)-\lambda w^j(x_t)\big)
= -\lambda\big(w^i(x_t)-w^j(x_t)\big).
\end{align}
Since $\lambda>0$, this implies
\begin{align}
r^i(x_t)-\lambda w^i(x_t) > r^j(x_t)-\lambda w^j(x_t)
\Longleftrightarrow
w^i(x_t) < w^j(x_t).
\end{align}
Substituting $w^i(x_t)=W_c(\nu_t,\mu_i)$ completes the result.
\end{proof}
\begin{lemma}[Concentration of frailty-adjusted rewards]
\label{lem:frailty_concentration}

Assume \textnormal{(A3)--(A4)}. Fix \(x_t\in\mathcal X\) and \(a_i\in\mathcal A\), and define the frailty reward as
$R_t^i
:=
\delta_t(i)\,
S_i(T_t(i)\mid x_t)^{\theta_t}$,
where \(0\le S_i(\tau\mid x_t)\le1\), \(\theta_t>0\), and
\(\delta_t(i)\in\{0,1\}\). Then, the centered reward $R_t^i-\mathbb E[R_t^i\mid x_t]$
is conditionally sub-Gaussian. More precisely, for every
\(\varepsilon>0\),
\[
\mathbb P\!\left(
\left|
R_t^i-\mathbb E[R_t^i\mid x_t]
\right|>\varepsilon
\,\middle|\,x_t
\right)
\le
2e^{-2\varepsilon^2}.
\]

\end{lemma}
\begin{proof}

Since \(0 \le S_i(T_t(i)\mid x_t) \le 1\) and \(\delta_t(i)\in\{0,1\}\), we have
\[
0 \le R_t^i \le 1 \qquad \text{almost surely}.
\]
Fix \(x_t\) and define
\[
m := \mathbb{E}[R_t^i \mid x_t], 
\qquad
X := R_t^i - m.
\]
Then \(m \in [0,1]\) and \(\mathbb{E}[X \mid x_t]=0\).

We first bound the conditional moment generating function. For any \(\lambda \in \mathbb{R}\), convexity of \(y \mapsto e^{\lambda y}\) on \([0,1]\) implies that for every \(y \in [0,1]\),
\[
e^{\lambda y} \le (1-y)e^{0} + y e^{\lambda} = 1 - y + y e^{\lambda}.
\]
Applying this pointwise inequality to \(R_t^i\) and taking conditional expectation yields
\[
\mathbb{E}[e^{\lambda R_t^i} \mid x_t]
\le 1 + m(e^{\lambda}-1).
\]
Therefore,
\[
\mathbb{E}[e^{\lambda X} \mid x_t]
=
e^{-\lambda m}\mathbb{E}[e^{\lambda R_t^i} \mid x_t]
\le
e^{-\lambda m}\bigl(1 + m(e^{\lambda}-1)\bigr)
=: \phi(\lambda,m).
\]

Define \(\psi(\lambda) := \log \phi(\lambda,m)\). A direct computation gives
\[
\psi(0)=0, 
\qquad 
\psi'(0)=0,
\qquad
\psi''(\lambda)
=
\frac{m(1-m)e^{\lambda}}{(1-m+me^{\lambda})^2}.
\]
A direct simplification shows that
\[
\psi''(\lambda)
=
\frac{1}{4\cosh^2(\lambda/2)}
\le \frac{1}{4}
\qquad \text{for all } \lambda \in \mathbb{R},\, m \in [0,1].
\]

Now apply Taylor’s theorem with integral remainder:
\[
\psi(\lambda)
=
\psi(0) + \psi'(0)\lambda + \int_0^\lambda (\lambda-s)\psi''(s)\,ds.
\]
Using \(\psi(0)=\psi'(0)=0\) and \(\psi''(s)\le \tfrac14\), we obtain
\[
\psi(\lambda)
\le
\int_0^\lambda (\lambda-s)\frac{1}{4}\,ds
=
\frac{\lambda^2}{8}.
\]
Hence,
\[
\mathbb{E}[e^{\lambda X} \mid x_t]
\le
\exp\!\left(\frac{\lambda^2}{8}\right),
\]
so \(X\) is conditionally sub-Gaussian with variance proxy \(1/4\).

Finally, for any \(\lambda>0\), Markov’s inequality gives
\[
\mathbb{P}(X>\varepsilon \mid x_t)
\le
\exp\!\left(-\lambda \varepsilon + \frac{\lambda^2}{8}\right).
\]
Optimizing in \(\lambda\) yields \(\lambda=4\varepsilon\), hence
\[
\mathbb{P}(X>\varepsilon \mid x_t)
\le
\exp(-2\varepsilon^2).
\]
Applying the same argument to \(-X\) yields
\[
\mathbb{P}(|X|>\varepsilon \mid x_t)
\le
2e^{-2\varepsilon^2}.
\]

\end{proof}
\subsection{Proof of \ref{item:robustness} in Theorem \ref{thm:main_ot_bandits} (Margin robustness under Gaussian noise)}
\label{appendix:margin}

\begin{lemma}[Margin robustness under Gaussian noise]\label{lem:margin}
Let $\widetilde{W}_c(\mu_i,\nu)=W_c(\mu_i,\nu)+\epsilon_i$, where $\epsilon_i \sim \mathcal{N}(0,\sigma^2)$ are independent. Define for $\lambda>0$, $\Delta_{ij} := \lambda (W_c(\mu_j,\nu) - W_c(\mu_i,\nu))$. If $\Delta_{ij} > 0$, then the probability of incorrect ordering satisfies
\[
\mathbb{P}\big(r^i(x_t) < r^j(x_t)\big)
=
\Phi\!\left(-\frac{\Delta_{ij}}{\sqrt{2}\sigma}\right)
\le \frac{1}{2}\exp\!\left(-\frac{\Delta_{ij}^2}{4\sigma^2}\right).
\]
In particular, if $\Delta_{ij} \ge \sigma \sqrt{2\log 2}$, then
\[
\mathbb{P}\big(r^i(x_t) < r^j(x_t)\big) \le \tfrac{1}{4}.
\]
\end{lemma}

\begin{proof}
Without loss of generality assume $\Delta_{ij} > 0$. The noisy difference satisfies
\[
\widetilde{\Delta}_{ij}
=
\lambda(\widetilde{W}_c(\mu_j,\nu)-\widetilde{W}_c(\mu_i,\nu))
=
\Delta_{ij} + Z,
\]
where $Z=\epsilon_j-\epsilon_i \sim \mathcal{N}(0,2\sigma^2)$. A misranking occurs when $r^i(x_t) < r^j(x_t)$, which is equivalent to $\widetilde{\Delta}_{ij} < 0$, i.e. $Z < -\Delta_{ij}$.
Hence
\[
\mathbb{P}(r^i(x_t) < r^j(x_t))
=
\mathbb{P}(Z < -\Delta_{ij})
=
\Phi\!\left(-\frac{\Delta_{ij}}{\sqrt{2}\sigma}\right).
\]
Using the Gaussian tail bound $\Phi(-x)\le \frac{1}{2}e^{-x^2/2}$ for $x\ge 0$, we obtain
\[
\mathbb{P}(r^i(x_t) < r^j(x_t))
\le
\frac{1}{2}\exp\!\left(-\frac{\Delta_{ij}^2}{4\sigma^2}\right).
\]
The condition $\Delta_{ij} \ge \sigma \sqrt{2\log 2}$ ensures the right-hand side is at most $1/4$, completing the proof.
\end{proof}
\subsection{Proof of \ref{item:convergence_weights} in Theorem \ref{thm:main_ot_bandits} (Convergence of orchestration weights)}
\label{appendix:convergence}

\begin{theorem}[Almost-sure convergence of orchestration weights]
Assume (A1)--(A5). Let $(\gamma_t)_{t\ge0}$ satisfy
\[
\gamma_t > 0,\qquad 
\sum_{t=0}^\infty \gamma_t = \infty,\qquad 
\sum_{t=0}^\infty \gamma_t^2 < \infty.
\]
Assume the task process $(x_t)$ is stationary and independent of $(\phi_t)$, and that there exists a measurable function $\bar r:\mathcal X \to \mathbb{R}^M$ such that
\[
\mathbb{E}[R_t^i - \lambda W_t^i \mid x_t] = \bar r^i(x_t).
\]
Then $(\phi_t)$ converges almost surely to the internally chain transitive set of the ODE
\[
\dot{\phi} = \mathrm{Softmax}(\bar r(x)) - \phi,
\]
where $\bar r(x)$ is evaluated under the stationary distribution of $x_t$.
\end{theorem}
\begin{proof}
Let $\mathcal F_t = \sigma(\phi_s, x_s, R_s, W_s : s \le t)$.

The update can be written as
\[
\phi_{t+1}
=
\phi_t + \gamma_t \bigl(\mathrm{Softmax}(u_t) - \phi_t\bigr),
\]
where $u_t := (R_t^i - \lambda W_t^i)_{i=1}^M$.

Define the drift:
\[
H(\phi_t)
:=
\mathbb{E}[\mathrm{Softmax}(u_t) \mid \mathcal F_{t-1}] - \phi_t,
\]
and the noise term:
\[
\xi_{t+1}
:=
\mathrm{Softmax}(u_t)
-
\mathbb{E}[\mathrm{Softmax}(u_t) \mid \mathcal F_{t-1}].
\]

Then
\[
\phi_{t+1} = \phi_t + \gamma_t \big(H(\phi_t) + \xi_{t+1}\big).
\]

By boundedness of rewards, $\mathrm{Softmax}(u_t)$ takes values in the simplex, hence $\xi_{t+1}$ is a martingale difference sequence with bounded second moments.

Now use the stationarity assumption: since $x_t$ is stationary and independent of $\phi_t$,
\[
\mathbb{E}[\mathrm{Softmax}(u_t) \mid \mathcal F_{t-1}]
=
\mathbb{E}_{x \sim \pi_x}\big[\mathrm{Softmax}(\bar r(x))\big],
\]
where $\pi_x$ is the invariant distribution of $(x_t)$.

Define the averaged drift
\[
\bar H(\phi)
:=
\mathbb{E}_{x \sim \pi_x}\big[\mathrm{Softmax}(\bar r(x))\big] - \phi.
\]

Then the recursion becomes a standard Robbins--Monro scheme:
\[
\phi_{t+1}
=
\phi_t + \gamma_t\big(\bar H(\phi_t) + \xi_{t+1} + o(1)\big).
\]

The function $\bar H$ is Lipschitz on the simplex, and the noise term has bounded second moment. The step-size conditions ensure
\[
\sum_t \gamma_t = \infty,
\qquad
\sum_t \gamma_t^2 < \infty,
\]
so the stochastic approximation framework applies.

By the stochastic approximation theorem of Benaïm (1999), the piecewise linear interpolation of $(\phi_t)$ is an asymptotic pseudo-trajectory of the ODE
\[
\dot \phi = \bar H(\phi).
\]

Therefore, the limit set of $(\phi_t)$ is contained in the internally chain transitive invariant sets of this ODE.

Since the ODE is a linear contraction toward the fixed point
\[
\phi^\ast = \mathbb{E}_{x \sim \pi_x}[\mathrm{Softmax}(\bar r(x))],
\]
all internally chain transitive sets reduce to this equilibrium, which is globally asymptotically stable.

Hence,
\[
\phi_t \xrightarrow{a.s.} \phi^\ast.
\]
\end{proof}
\begin{remark}
 The proof follows a stochastic approximation argument in which the orchestration weights evolve as a Robbins--Monro recursion driven by a softmax response to OT-adjusted rewards. The key step is the decomposition of the update into a deterministic drift term and a martingale difference noise, allowing the dynamics to be linked to a limiting ordinary differential equation via the asymptotic pseudo-trajectory framework. Intuitively, the softmax operator induces a smooth, entropy-regularised selection mechanism that continuously reweights agents according to their relative OT-adjusted performance. Under step sizes satisfying the usual summability conditions, stochastic fluctuations vanish asymptotically, and the evolution tracks the averaged response of this softmax mechanism under the stationary task distribution.

The main structural novelty lies in the fact that the drift is not driven by exogenous rewards but by optimal-transport-adjusted utilities, which couple agent performance through a geometric cost. This embeds a stochastic approximation scheme into a geometry-aware, mean-field decision system, where the limiting ODE describes an entropy-regularised flow toward stationary softmax equilibria. 
\end{remark}

\subsection{Proof of \ref{item:consistency_weights} in Theorem \ref{thm:main_ot_bandits} (Consistency of reward estimates)}
\label{appendix:consistency}

\begin{theorem}[Empirical reward decomposition and consistency]\label{thm:consistency}
Assume (A2)--(A4) and bounded rewards $|R_t^i|\le R_{\max}$. Define
\[
r^i(x_t) := \mathbb{E}[R_t^i \mid x_t],
\qquad
\hat r_t(i) := \frac{1}{t}\sum_{s=1}^t R_s^i.
\]

Then the following hold:

\textnormal{(i) i.i.d. case.}
If $(x_t, \mathbf{R}_t)_{t\ge1}$ are i.i.d., then for every $i \in \mathcal{A}$,
\[
\hat r_t(i) \xrightarrow[t\to\infty]{\mathrm{a.s.}} \mathbb{E}[R_t^i].
\]

\textnormal{(ii) general (possibly non-i.i.d.) case.}
Let $\mathcal{H}_{t-1}$ be the filtration generated by $(x_s, \mathbf{R}_s)_{s\le t-1}$, and define the martingale difference sequence $\xi_t(i) := R_t^i - \mathbb{E}[R_t^i \mid \mathcal{H}_{t-1}]$. Assume $\mathbb{E}[\xi_t(i)\mid \mathcal{H}_{t-1}] = 0$ and $|\xi_t(i)|\le R_{\max}$ a.s. Then, for every $i\in\mathcal{A}$,
\[
\hat r_t(i) - \frac{1}{t}\sum_{s=1}^t \mathbb{E}[R_s^i \mid \mathcal{H}_{s-1}]
\xrightarrow[t\to\infty]{\mathrm{a.s.}} 0.
\]
Moreover, since $\mathcal{A}$ is finite,
\[
\sup_{i\in\mathcal{A}}
\left|
\hat r_t(i) - \frac{1}{t}\sum_{s=1}^t \mathbb{E}[R_s^i \mid \mathcal{H}_{s-1}]
\right|
\xrightarrow[t\to\infty]{\mathrm{a.s.}} 0.
\]

\end{theorem}
\begin{proof}

We treat the two cases separately.

\paragraph{(i) i.i.d. case.}
If $(x_t,\mathbf{R}_t)$ are i.i.d., then $(R_t^i)_{t\ge1}$ is an i.i.d. integrable sequence. Hence, by the strong law of large numbers,
\[
\hat r_t(i)
=
\frac{1}{t}\sum_{s=1}^t R_s^i
\xrightarrow[t\to\infty]{\mathrm{a.s.}}
\mathbb{E}[R_t^i].
\]

\paragraph{(ii) general non-i.i.d. case.}

Let $\mathcal{H}_t := \sigma(x_s,\mathbf{R}_s: s\le t)$ and write
\[
R_t^i = \mathbb{E}[R_t^i \mid \mathcal{H}_{t-1}] + \xi_t(i),
\qquad
\xi_t(i) := R_t^i - \mathbb{E}[R_t^i \mid \mathcal{H}_{t-1}].
\]

Then $(\xi_t(i),\mathcal{H}_t)$ is a martingale difference sequence with
\[
\mathbb{E}[\xi_t(i)\mid \mathcal{H}_{t-1}] = 0,
\qquad
\mathbb{E}[\xi_t(i)^2 \mid \mathcal{H}_{t-1}] \le R_{\max}^2.
\]

Summing over $t$ gives
\[
\hat r_t(i)
=
\frac{1}{t}\sum_{s=1}^t \mathbb{E}[R_s^i \mid \mathcal{H}_{s-1}]
+
\frac{1}{t}\sum_{s=1}^t \xi_s(i).
\]

Hence,
\[
\hat r_t(i)
-
\frac{1}{t}\sum_{s=1}^t \mathbb{E}[R_s^i \mid \mathcal{H}_{s-1}]
=
\frac{1}{t}\sum_{s=1}^t \xi_s(i).
\]

Since
\[
\sum_{t=1}^\infty \frac{\mathbb{E}[\xi_t(i)^2 \mid \mathcal{H}_{t-1}]}{t^2}
\le
R_{\max}^2 \sum_{t=1}^\infty \frac{1}{t^2}
< \infty,
\]
the martingale strong law implies
\[
\frac{1}{t}\sum_{s=1}^t \xi_s(i)
\xrightarrow[t\to\infty]{\mathrm{a.s.}} 0.
\]

Therefore,
\[
\hat r_t(i)
-
\frac{1}{t}\sum_{s=1}^t \mathbb{E}[R_s^i \mid \mathcal{H}_{s-1}]
\xrightarrow[t\to\infty]{\mathrm{a.s.}} 0.
\]

Finally, since $\mathcal{A}$ is finite, taking a maximum preserves almost sure convergence:
\[
\max_{i\in\mathcal{A}}
\left|
\frac{1}{t}\sum_{s=1}^t \xi_s(i)
\right|
\xrightarrow[t\to\infty]{\mathrm{a.s.}} 0.
\]
\end{proof}
\begin{remark}
In the i.i.d. case, the empirical average $\hat r_t(i)$ is formed from independent samples drawn from a fixed distribution, so the strong law of large numbers implies convergence to the constant mean $\mathbb{E}[R_t^i]$.

In the non-i.i.d. case, the rewards are adapted to the history and admit the decomposition into a predictable component $\mathbb{E}[R_t^i \mid \mathcal{H}_{t-1}]$ and a martingale difference noise term. Consequently, the empirical average splits into the time-average of the predictable process and a martingale average, and boundedness ensures that the latter vanishes almost surely. This yields that $\hat r_t(i)$ tracks the time-averaged conditional expectation
\[
\frac{1}{t}\sum_{s=1}^t \mathbb{E}[R_s^i \mid \mathcal{H}_{s-1}],
\]
up to an asymptotically negligible martingale error. However, full convergence of $\hat r_t(i)$ is not guaranteed in general, since the predictable process $\mathbb{E}[R_t^i \mid \mathcal{H}_{t-1}]$ may itself fail to converge or even fail to admit a Ces\`aro limit (for instance, under oscillatory or adversarial dynamics). Additional structure, such as ergodicity or asymptotic stationarity of the conditional mean process, is required to ensure convergence of the empirical averages. 
\end{remark}

\section{Visualization of the Synthetic Datasets}

\label{sec:semi_synthetic_appednix}
\begin{figure}[H]
    \centering
    \includegraphics[width=.8\linewidth]{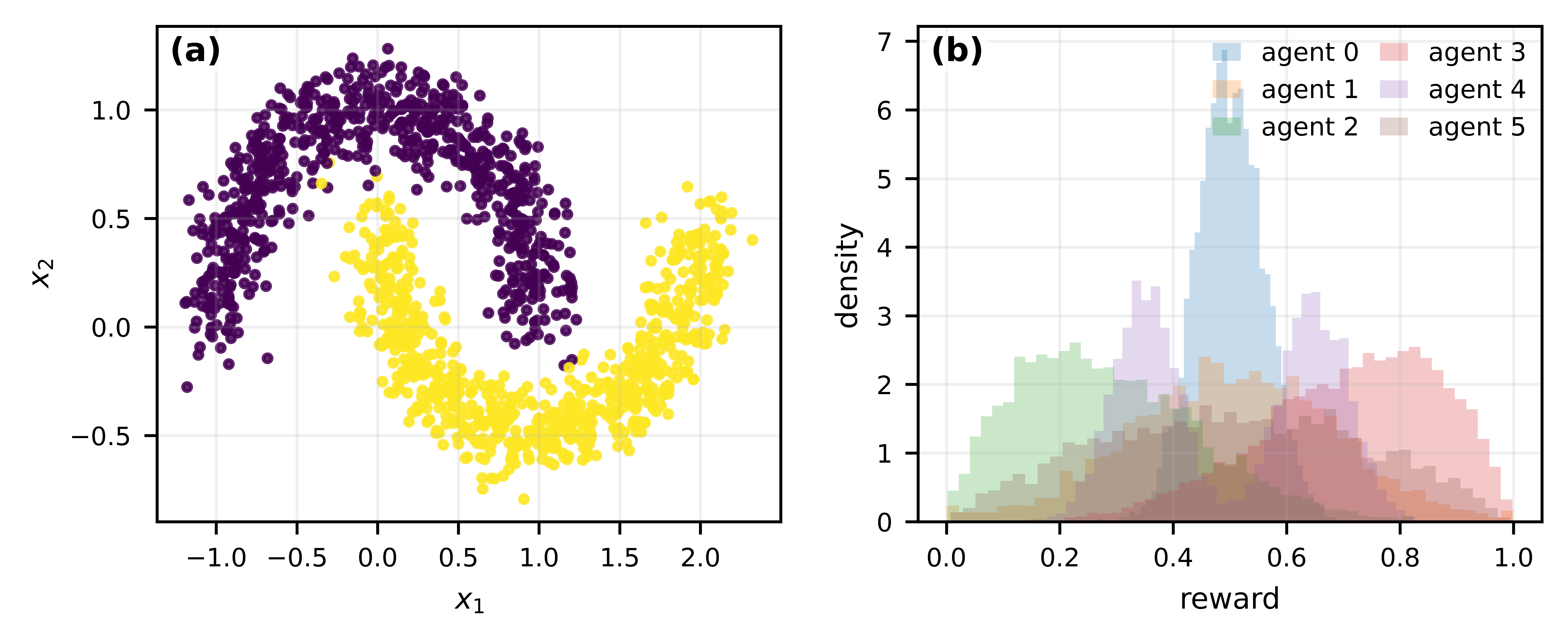}
     \caption{\textbf{IID data generation} 
(a) IID task contexts sampled from a half-moons distribution, illustrating a stationary but structured task space. 
(b) IID reward distributions for $K$ agents, where rewards are independently drawn over time from fixed distributions with a matched mean (approximately 0.5) but heterogeneous higher-order properties (e.g., variance, skewness, and bimodality).}

    \label{fig:iid}
\end{figure}

\begin{figure}[H]
    \centering
    \includegraphics[width=0.8\linewidth]{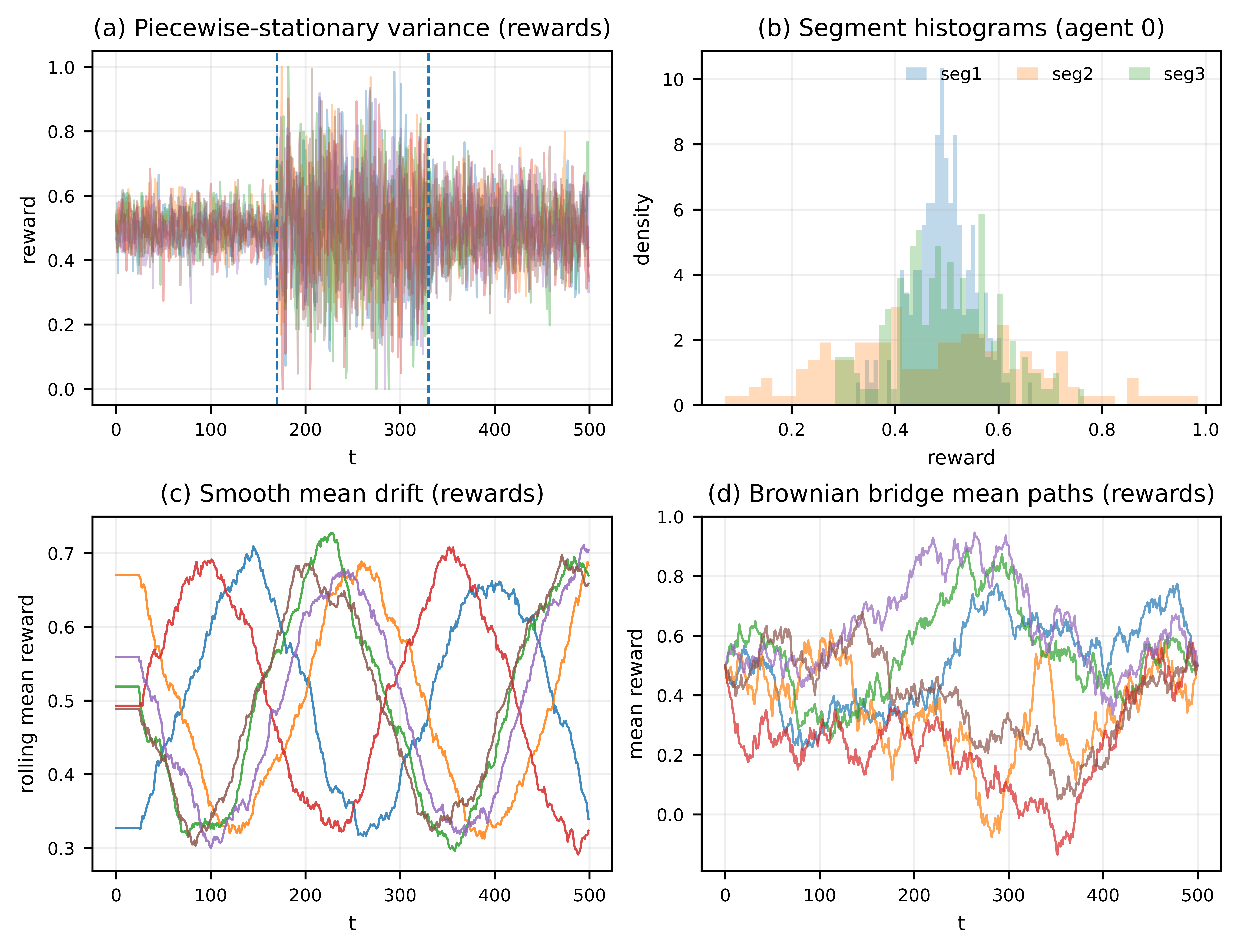}
    \caption{\textbf{Non-IID data generation}
    (a) Piecewise-stationary rewards where the variance changes at unknown changepoints while the mean remains fixed, shown as reward trajectories over time across agents. (b) Segment-wise reward histograms for a representative agent, revealing distributional shifts induced by variance changepoints. (c) Smooth-drift setting where agent reward means evolve gradually according to a sinusoidal drift (shown via rolling mean trajectories), inducing continuous non-stationarity. (d)Brownian-bridge setting illustrating temporally correlated latent mean paths constrained to fixed endpoints, producing structured stochastic dependence across time.}
    \label{fig:non_iid}
\end{figure}

\section{Additional Synthetic Experiments}
\label{add_suntehtic_experiments}

\subsection{Dataset and Task Description}
Datasets and tasks are as in \ref{suntehtic_experiments}.

\subsection{Baselines}
Baselines are as in Section \ref{suntehtic_experiments}.

\subsection{Evaluation Metrics}
%
\label{evaluation_metrics2}
\begin{enumerate}[label={$\bullet$},leftmargin=*,itemsep=1pt,topsep=1pt]
  \item \textbf{Event rate.}
  Under survival-style feedback with censoring indicator $\delta_t(i_t)\in\{0,1\}$, we report $\frac{1}{T}\sum_{t=1}^T \delta_t(i_t)$, i.e., the fraction of rounds with uncensored (fully observed) outcomes. Higher event rates indicate less censoring and more informative feedback.
  \item \textbf{Mean observed time.}
  We report the mean observed time $\frac{1}{T}\sum_{t=1}^T T^{\mathrm{obs}}_t(i_t)$, where\\$T^{\mathrm{obs}}_t(i_t)=\min\{T_t(i_t), C_t(i_t)\}$ under right censoring. This metric summarizes the typical observed completion time under the censoring mechanism.
\end{enumerate}
%
%
%
\begin{table*}[t]
\centering
\scriptsize
\setlength{\tabcolsep}{3pt}
\makebox[\textwidth][c]{
\begin{tabular}{|l|cccc|cccc|}
\hline
& \multicolumn{4}{c|}{\textbf{Event Rate}} & \multicolumn{4}{c|}{\textbf{Mean Observed Time}} \\
\hline
Environment & BOT-Orch & No-OT ($\lambda$=0) & Random & UCB1 (MAB) & BOT-Orch & No-OT ($\lambda$=0) & Random & UCB1 (MAB) \\
\hline
\textbf{IID} & \multicolumn{4}{c|}{} & \multicolumn{4}{c|}{} \\
IID-G & \textbf{0.63$\pm$0.02} & 0.60$\pm$0.02 & 0.58$\pm$0.04 & 0.54$\pm$0.03 & \textbf{0.59$\pm$0.04} & 0.66$\pm$0.04 & 0.72$\pm$0.04 & 0.71$\pm$0.03 \\
IID-M & \textbf{0.65$\pm$0.04} & 0.60$\pm$0.04 & 0.60$\pm$0.03 & 0.59$\pm$0.03 & \textbf{0.58$\pm$0.02} & 0.66$\pm$0.04 & 0.69$\pm$0.03 & 0.70$\pm$0.05 \\
\hline
\textbf{Non-IID} & \multicolumn{4}{c|}{} & \multicolumn{4}{c|}{} \\
NonIID-BB & \textbf{0.67$\pm$0.04} & 0.61$\pm$0.03 & 0.61$\pm$0.02 & 0.62$\pm$0.05 & \textbf{0.55$\pm$0.09} & 0.62$\pm$0.01 & 0.65$\pm$0.04 & 0.63$\pm$0.03 \\
NonIID-PS & \textbf{0.66$\pm$0.04} & 0.64$\pm$0.02 & 0.59$\pm$0.04 & 0.59$\pm$0.02 & \textbf{0.56$\pm$0.04} & 0.62$\pm$0.04 & 0.67$\pm$0.03 & 0.69$\pm$0.04 \\
NonIID-SD & \textbf{0.65$\pm$0.03} & 0.64$\pm$0.03 & 0.60$\pm$0.04 & 0.57$\pm$0.02 & \textbf{0.59$\pm$0.03} & 0.64$\pm$0.05 & 0.67$\pm$0.03 & 0.67$\pm$0.04 \\
\hline
\end{tabular}}
\caption{Event Rate and Mean Observed Time (mean \(\pm\) 95\% CI across 5 seeds; \(T=200\)).}
\label{tab:event_rate_ci_4col}
\end{table*}

\subsection{Results.}
Table~\ref{tab:event_rate_ci_4col} reports event rate and mean observed time. BOT-Orch achieves the highest event rate and lowest mean observed time across all environments, in both IID and non-IID settings, consistently outperforming No-OT, Random, and UCB1.

\subsection{Ablation Study: Sensitivity to the Alignment
Penalty \texorpdfstring{$\lambda$}{lambda}}
\label{sec:ablation_synthetic}

The alignment penalty weight $\lambda$ controls the
trade-off between exploitation of historical reward
estimates $\rhat_t(i)$ and adherence to OT-based
distributional alignment $W_t(i)$.
We conduct a grid search over
$\lambda \in \{0.0,\,0.5,\,1.0,\,1.25,\,1.5,\,1.75,\,
2.0,\,3.0,\,5.0,\,10.0\}$,
running BOT-Orch for 30 seeds under both the IID
(Algorithm~1) and Non-IID (Algorithm~2) conditions.
The No-OT, Random, and UCB1 baselines serve as fixed
reference lines since they are independent of $\lambda$.
Table~\ref{tab:lambda_ablation_synthetic} summarise the results.

\begin{table}[t]
\centering
\scriptsize
\setlength{\tabcolsep}{4pt}
\begin{tabular}{|l|ccc|ccc|}
\hline
& \multicolumn{3}{c|}{IID (Algorithm 1)} & \multicolumn{3}{c|}{Non-IID (Algorithm 2)}\\
\hline
$\lambda$ & Cum. net & Regret & Team acc. & Cum. net & Regret & Team acc.\\
\hline
0 & -311.34{$\pm$3.43} & 122.61{$\pm$2.49} & 0.177$\pm$0.013 & -722.67$\pm$47.12 & 317.61$\pm$21.59 & 0.177$\pm$0.013 \\
0.5 & -271.42$\pm$2.79 & 82.68$\pm$1.92 & 0.336$\pm$0.011 & -549.11$\pm$33.07 & 144.05$\pm$5.85 & 0.418$\pm$0.017 \\
1 & -262.29$\pm$2.69 & 73.56$\pm$1.78 & 0.380$\pm$0.011 & -531.17$\pm$31.69 & 126.11$\pm$4.86 & 0.453$\pm$0.017 \\
1.25 & -259.93$\pm$2.57 & 71.19$\pm$1.78 & 0.386$\pm$0.011 & -528.55$\pm$31.48 & 123.49$\pm$5.23 & 0.459$\pm$0.019 \\
1.5 & -258.52$\pm$2.62 & 69.78$\pm$1.72 & 0.389$\pm$0.011 & -526.24$\pm$31.88 & 121.18$\pm$5.18 & 0.462$\pm$0.018 \\
1.75 & -257.98$\pm$2.27 & 69.24$\pm$1.47 & 0.398$\pm$0.009 & -524.10$\pm$31.77 & 119.04$\pm$5.49 & 0.470$\pm$0.019 \\
2 & -256.20$\pm$2.55 & 67.47$\pm$1.64 & 0.403$\pm$0.010 & -524.22$\pm$31.76 & 119.16$\pm$5.05 & 0.469$\pm$0.018 \\
3 & -253.76$\pm$2.49 & 65.02$\pm$1.60 & 0.412$\pm$0.011 & -519.79$\pm$31.65 & 114.72$\pm$4.79 & 0.479$\pm$0.019 \\
5 & -253.28$\pm$2.04 & 64.54$\pm$1.67 & 0.413$\pm$0.010 & -515.97$\pm$32.03 & 110.91$\pm$5.11 & 0.484$\pm$0.020 \\
10 & -251.99$\pm$2.16 & 63.25$\pm$1.46 & 0.420$\pm$0.009 & -514.23$\pm$31.32 & 109.16$\pm$4.67 & 0.493$\pm$0.018 \\
11 & -252.06$\pm$2.22 & 63.32$\pm$1.47 & 0.420$\pm$0.009 & -513.61$\pm$31.26 & 108.54$\pm$4.71 & 0.495$\pm$0.018 \\
12 & -252.05$\pm$2.28 & 63.32$\pm$1.51 & 0.420$\pm$0.009 & -512.93$\pm$31.33 & 107.86$\pm$4.73 & 0.495$\pm$0.018 \\
13 & \textbf{-251.60$\pm$2.33} & \textbf{62.86$\pm$1.60} & \textbf{0.422$\pm$0.009} & \textbf{-513.01$\pm$31.01} & \textbf{107.95$\pm$5.00} & \textbf{0.494$\pm$0.018} \\
14 & -251.71$\pm$2.42 & 62.97$\pm$1.50 & 0.423$\pm$0.009 & -513.16$\pm$30.99 & 108.10$\pm$5.07 & 0.494$\pm$0.018 \\
15 & -251.80$\pm$2.45 & 63.06$\pm$1.56 & 0.423$\pm$0.008 & -513.45$\pm$31.02 & 108.39$\pm$5.01 & 0.493$\pm$0.018 \\
\hline
\multicolumn{7}{|l|}{\textit{Reference baselines ($\lambda$-independent)}}\\
\hline
No-OT & -311.34$\pm$3.43 & 122.61$\pm$2.49 & 0.177$\pm$0.013 & -722.67$\pm$47.12 & 317.61$\pm$21.59 & 0.177$\pm$0.013 \\
Random & -315.75$\pm$3.74 & 127.02$\pm$3.11 & 0.177$\pm$0.010 & -731.00$\pm$50.00 & 325.93$\pm$22.38 & 0.169$\pm$0.010 \\
UCB1 & -313.22$\pm$3.92 & 124.48$\pm$2.93 & 0.177$\pm$0.012 & -722.86$\pm$48.17 & 317.79$\pm$21.51 & 0.178$\pm$0.009 \\
\hline
\end{tabular}
\caption{$\lambda$ grid search results for BOT-Orch on synthetic tasks. Mean $\pm$ 95\% CI across 30 seeds, $T=114$, evaluated with $\lambda_{eval}=1.0$. $\lambda=0$ reproduces No-OT by construction (sanity check). Selected $\lambda^*=13$ (bold) maximizes average Cum. net across IID/Non-IID panels. Reference baselines at the bottom are $\lambda$-independent.}
\label{tab:lambda_ablation_synthetic}
\end{table}
%
%

\section{Semi-Synthetic Experiment Settings}
\label{sec:app_setting}

\paragraph{Agents.} We set $M = 2$ agents:
\begin{itemize}[label={$\bullet$},leftmargin=*,itemsep=1pt,topsep=1pt]
  \item \textbf{Agent~0 (AI)}: a logistic regression
    classifier (L2, $C{=}1.0$) trained on the training split
    and calibrated using Platt scaling (isotonic regression)
    on the calibration split.
    Accuracy: 98.2\% in-distribution, 80.7\% under shift.
  \item \textbf{Agent~1 (Human)}: a simulated clinical expert
    with complementary accuracy
    88.0\% on in-distribution patients,
    94.7\% on shifted patients.
\end{itemize}
The human expert is more accurate on the patients the AI
handles worst, confirming a positive complementarity gap
$\Delta_{\mathrm{comp}} > 0$.

\paragraph{Tasks and reward.}
At each round $t$, a patient biopsy $x_t \in \X$ arrives.
The reward is binary correctness: $R_t(i) = \mathbf{1}[\text{agent } i
  \text{ classifies patient } t \text{ correctly}]$. This is bounded in $[0,1]$, satisfying Assumption~(A2).
We use bandit feedback throughout and only the chosen agent's
reward is observed.

\paragraph{OT alignment costs.}
We use the output-space OT alignment cost on the binary
label simplex $\{0,1\}$ with ground cost
$$C = \bigl[\begin{smallmatrix}0&1\\1&0\end{smallmatrix}\bigr]$$
(0-1 loss).
Under this cost, the Wasserstein distance between the true
label one-hot $\nu_t = \delta_{y_t}$ and each agent's
predictive distribution admits the closed form:
\begin{align}
  W_t(\text{AI})
  &= 1 - R_t(\text{AI})
  \quad\text{(probability AI is wrong on patient }t\text{)},
  \label{eq:W_ai}\\
  W_t(\text{human})
  &= 1 - p_h(x_t)
  \quad\text{(probability human is wrong on patient }t\text{)},
  \label{eq:W_human}
\end{align}
where $p_h(x_t)$ is the human's accuracy for that patient's
shift status.
This is a direct instantiation of the general alignment cost
$W_c(\nu_t, \mu_i)$ defined in Section~3.4.
The closed form follows from the exactness of Sinkhorn
transport on the $2\times 2$ binary simplex.

\paragraph{OT Dominance verification.}
With these costs, under distribution shift:
\[
  W_t(\text{AI}) \approx 0.193
  \quad\text{vs.}\quad
  W_t(\text{human}) \approx 0.053.
\]
The gap $W_t(\text{AI}) - W_t(\text{human}) = 0.140$
is large and positive on shifted patients, confirming the
precondition of Theorem~4.2 Part~2 that the human has strictly
lower alignment cost on shifted patients and should be
preferred by the BOT-Orch policy.
On in-distribution patients the relationship is reversed
($W_t(\text{AI}) \approx 0.018 < 0.120 = W_t(\text{human})$),
so the AI is correctly preferred there.
This is exactly the complementarity structure that
BOT-Orch is designed to exploit without being told about
the shift.

\section{Additional Semi-Synthetic Experiments and Figures}
\label{additional_results_Med}

\subsection{Diagnostic Analysis}
Figure~\ref{fig:escalation} and ~\ref{fig:agent_selection} provide diagnostic analyses.

\begin{figure}[H]
\centering
\includegraphics[width=0.88\linewidth]{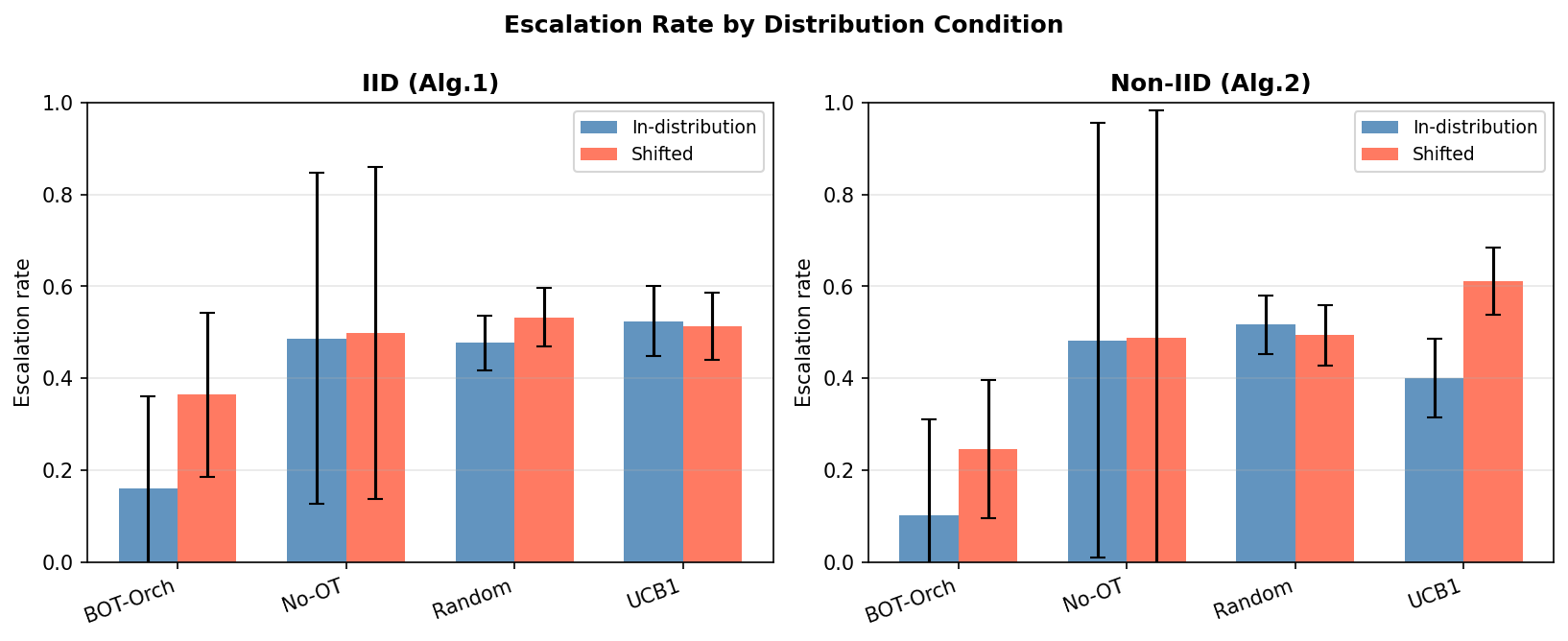}
\caption{%
  \textbf{Escalation rate by distribution condition.}
  Mean escalation rate for in-distribution patients
  (blue) and shifted patients (red) per method.
  Error bars show standard deviation across seeds.
  \textit{Left}: IID condition;
  \textit{right}: Non-IID condition.
  BOT-Orch achieves a higher escalation rate on shifted
  patients relative to in-distribution patients compared
  to all baselines, demonstrating targeted routing.
  No-OT's large error bars reflect bimodal behaviour
  across seeds: some runs converge to always-AI and others
  to always-human, due to the cold-start problem in the
  binary bandit without OT regularisation.}
\label{fig:escalation}
\end{figure}

\begin{figure}[H]
\centering
\includegraphics[width=\linewidth]{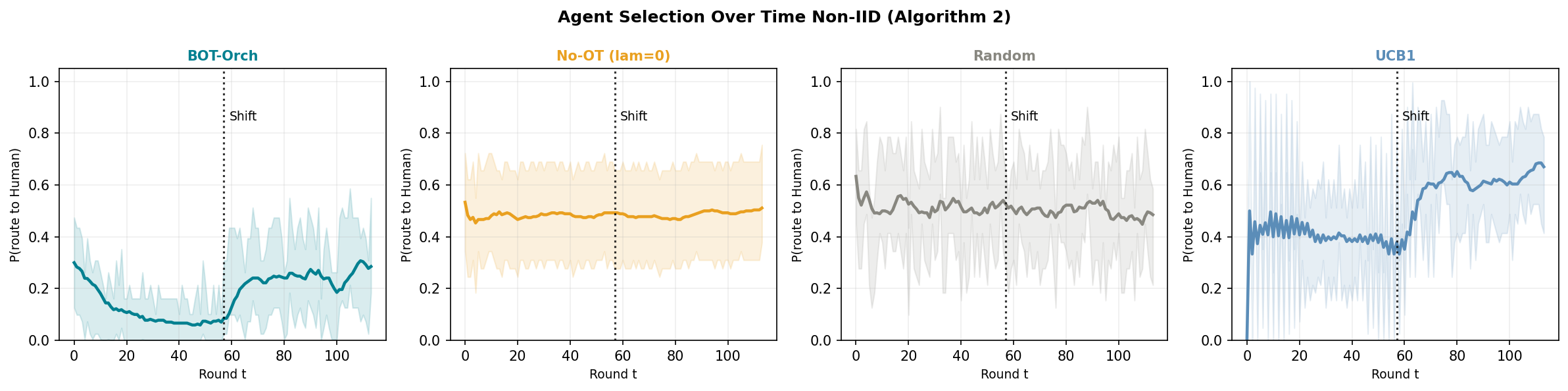}
\caption{%
  \textbf{Agent selection trajectories, Non-IID condition
  (Algorithm~2).}
  Rolling probability of routing to the human expert
  (window $w{=}8$ rounds) for each method over the
  deployment episode.
  The dotted vertical line marks the shift onset at
  round~57.
  BOT-Orch (top-left): routing probability rises after the
  shift, showing adaptation.
  No-OT (top-right): high variance, consistent with the
  bimodal cold-start behaviour observed in
  Table~\ref{tab:combined_results}.
  UCB1 (bottom-right): also adapts but less precisely,
  routing more patients to the human without the
  alignment-cost targeting of BOT-Orch.
  Random (bottom-left): flat at 50\% throughout as
  expected.}
\label{fig:agent_selection}
\end{figure}

\subsection{Escalation Rate and Escalation Rate on Shifted Patients}
We present the escalation rate and escalation rate on shifted patients in Table \ref{tab:combined_results2}
\begin{table*}[ht]
\centering
\scriptsize
\setlength{\tabcolsep}{2pt}
\makebox[\textwidth][c]{
\begin{tabular}{|l|cccc|cccc|}
\hline
& \multicolumn{4}{c|}{\textbf{Escalation Rate}}
& \multicolumn{4}{c|}{\textbf{Escalation Rate on Shifted Patients}} \\
\hline
& BOT-Orch & No-OT ($\lambda{=}0$) & Random & UCB1
& BOT-Orch & No-OT ($\lambda{=}0$) & Random & UCB1 \\
\hline
\textbf{IID}
  & 0.214{$\pm$0.066}
  & 0.493{$\pm$0.360}
  & 0.505{$\pm$0.046}
  & 0.519{$\pm$0.061}
  & 0.283{$\pm$0.085}
  & 0.499{$\pm$0.362}
  & 0.533{$\pm$0.064}
  & 0.514{$\pm$0.073} \\
\textbf{Non-IID}
  & 0.192{$\pm$0.023}
  & 0.485{$\pm$0.484}
  & 0.505{$\pm$0.046}
  & 0.506{$\pm$0.062}
  & 0.209{$\pm$0.019}
  & 0.488{$\pm$0.495}
  & 0.494{$\pm$0.065}
  & 0.612{$\pm$0.074} \\
\hline
\end{tabular}
}
\caption{%
  Deployment metrics across all four methods and both experimental
  conditions. Mean $\pm$ 95\% CI across 30 seeds, $T{=}114$,
  $\lambda{=}3.0$.
  \emph{Esc.\ rate}: fraction of patients routed to human.
  \emph{Esc.(shift)}: escalation rate on shifted patients only.
  IID condition uses Algorithm~1; Non-IID uses Algorithm~2
  (ID patients rounds 1--57, shifted rounds 58--114).}
\label{tab:combined_results2}
\end{table*}

\subsection{Ablation Study: Sensitivity to the Alignment
Penalty \texorpdfstring{$\lambda$}{lambda}}
\label{sec:ablation}

The alignment penalty weight $\lambda$ controls the
trade-off between exploitation of historical reward
estimates $\rhat_t(i)$ and adherence to OT-based
distributional alignment $W_t(i)$.
We conduct a grid search over
$\lambda \in \{0.0,\,0.5,\,1.0,\,1.25,\,1.5,\,1.75,\,
2.0,\,3.0,\,5.0,\,10.0\}$,
running BOT-Orch for 30 seeds under both the IID
(Algorithm~1) and Non-IID (Algorithm~2) conditions.
The No-OT, Random, and UCB1 baselines serve as fixed
reference lines since they are independent of $\lambda$.
Table~\ref{tab:ablation} and
Figure~\ref{fig:lambda_ablation} summarise the results.

\begin{table}[H]
\centering
\scriptsize
\setlength{\tabcolsep}{4pt}
\begin{tabular}{|l|ccc|ccc|}
\hline
& \multicolumn{3}{c|}{\textbf{IID (Algorithm~1)}}
& \multicolumn{3}{c|}{\textbf{Non-IID (Algorithm~2)}} \\
\hline
$\lambda$
  & Cum.\ net & Regret & Team acc.
  & Cum.\ net & Regret & Team acc. \\
\hline
0.0
  & 103.37{$\pm$2.51} & 9.80{$\pm$1.61} & 0.907{$\pm$0.022}
  & 103.17{$\pm$1.80} & 10.14{$\pm$1.17} & 0.905{$\pm$0.016} \\
0.5
  & 102.83{$\pm$3.92} & 10.06{$\pm$3.50} & 0.935{$\pm$0.026}
  & 98.17{$\pm$2.73}  & 14.92{$\pm$2.06} & 0.908{$\pm$0.019} \\
1.0
  & 105.65{$\pm$5.57} & 7.03{$\pm$5.13}  & 0.964{$\pm$0.028}
  & 98.02{$\pm$6.26}  & 14.61{$\pm$6.10} & 0.930{$\pm$0.029} \\
1.25
  & 106.17{$\pm$6.04} & 6.33{$\pm$5.60}  & 0.970{$\pm$0.027}
  & 102.47{$\pm$6.83} & 9.98{$\pm$6.69}  & 0.955{$\pm$0.027} \\
1.5
  & 106.90{$\pm$5.37} & 5.29{$\pm$5.08}  & 0.975{$\pm$0.022}
  & 106.04{$\pm$6.21} & 5.91{$\pm$6.18}  & 0.970{$\pm$0.025} \\
1.75
  & 107.18{$\pm$5.74} & 4.96{$\pm$5.61}  & 0.979{$\pm$0.021}
  & 108.04{$\pm$5.57} & 4.08{$\pm$5.50}  & 0.981{$\pm$0.018} \\
2.0
  & 107.27{$\pm$5.95} & 4.48{$\pm$5.65}  & 0.980{$\pm$0.021}
  & 109.19{$\pm$5.53} & 2.71{$\pm$5.48}  & 0.986{$\pm$0.017} \\
\best{3.0}
  & \best{108.84{$\pm$2.22}} & \best{2.29{$\pm$2.11}} & \best{0.988{$\pm$0.010}}
  & \best{110.61{$\pm$1.03}} & \best{0.59{$\pm$0.93}} & \best{0.993{$\pm$0.007}} \\
5.0
  & 108.88{$\pm$1.25} & 1.01{$\pm$1.04}  & 0.993{$\pm$0.006}
  & 109.61{$\pm$0.99} & 0.37{$\pm$0.67}  & 0.995{$\pm$0.006} \\
10.0
  & 105.90{$\pm$1.01} & 0.26{$\pm$0.38}  & 0.993{$\pm$0.008}
  & 106.34{$\pm$0.73} & 0.10{$\pm$0.29}  & 0.995{$\pm$0.006} \\
\hline
\multicolumn{7}{|l|}{\textit{Reference baselines ($\lambda$-independent)}} \\
\hline
No-OT
  & 103.37{$\pm$2.51} & 9.80{$\pm$1.61}  & 0.907{$\pm$0.022}
  & 103.17{$\pm$1.80} & 10.14{$\pm$1.17} & 0.905{$\pm$0.016} \\
Random
  & 83.98{$\pm$5.92}  & 28.72{$\pm$6.03} & 0.917{$\pm$0.020}
  & 79.78{$\pm$5.52}  & 32.97{$\pm$4.87} & 0.905{$\pm$0.024} \\
UCB1
  & 80.55{$\pm$4.88}  & 31.31{$\pm$3.98} & 0.902{$\pm$0.026}
  & 85.82{$\pm$4.91}  & 26.26{$\pm$4.28} & 0.919{$\pm$0.025} \\
\hline
\end{tabular}
\vspace{6pt}
\caption{%
  \textbf{$\lambda$ grid search results for BOT-Orch.}
  Mean $\pm$ 95\% CI across 30 seeds, $T{=}114$.
  $\lambda{=}0$ reproduces No-OT exactly (sanity check:
  difference $= 0.000$).
  Optimal value $\lambda^*{=}3.0$ in \best{bold}.
  Reference baselines at the bottom are $\lambda$-independent.}
\label{tab:ablation}
\end{table}

\paragraph{Sensitivity and Optimal $\lambda$.}
In the IID condition (Figure~\ref{fig:lambda_ablation},
top row), cumulative net utility increases monotonically
from $\lambda{=}0$ through $\lambda{=}5.0$ before
declining at $\lambda{=}10.0$.
The difference between $\lambda{=}3.0$
($108.84\pm2.22$) and $\lambda{=}5.0$ ($108.88\pm1.25$)
is $0.04$ units --- well within one standard deviation ---
and the two values are statistically indistinguishable.

\paragraph{Phase Transition in Non-IID Settings.}
A qualitatively different pattern emerges in the Non-IID
condition (Figure~\ref{fig:lambda_ablation}, bottom row).
For $\lambda\in\{0.5,\,1.0,\,1.25\}$, BOT-Orch performs
\emph{worse} than No-OT, with net utility falling to as
low as $98.02$ at $\lambda{=}1.0$ against the No-OT
baseline of $103.17$ (red-shaded cells in
Table~\ref{tab:ablation}).
This counterintuitive dip arises because any positive
$\lambda$ penalises the human agent during the first 57
in-distribution rounds (where
$W_t(\text{human}){=}0.120 > W_t(\text{AI}){\approx}0$),
suppressing its EMA reward estimate through disuse.
When the distribution shifts at round~58, a small
$\lambda$ is insufficient to override the depressed human
reward history, and routing fails to redirect to the
human.
Performance recovers sharply above
$\lambda{\approx}1.5$, where the OT signal becomes strong
enough to dominate stale reward estimates immediately when
the shift arrives, and improves monotonically to the peak
at $\lambda^*{=}3.0$.

\paragraph{Degeneracy at Large $\lambda$.}
For $\lambda\geq5.0$, the OT penalty dominates the reward
signal entirely.
The policy degenerates toward a near-deterministic OT
routing rule, effectively ignoring learned reward history.
Two symptoms confirm this regime: (i) the standard
deviation of net utility collapses from $\pm5.53$ at
$\lambda{=}2.0$ to $\pm0.73$ at $\lambda{=}10.0$ in
Non-IID --- not because the policy is more stable but
because it is no longer exploring; and (ii) net utility
falls at $\lambda{=}10.0$ ($105.90$ IID, $106.34$
Non-IID) below the optimum at $\lambda^*{=}3.0$.

\paragraph{Identification of $\lambda^*{=}3.0$.}
We identify $\lambda^*{=}3.0$ as the joint optimum across
both conditions on the primary metric.
It achieves the peak net utility in the Non-IID condition
($110.61\pm1.03$), is statistically tied with
$\lambda{=}5.0$ in IID (gap $0.04 <$ pooled std $1.73$),
reduces variance by $3\times$ relative to $\lambda{=}2.0$,
and is the last value at which bandit learning and OT
alignment both contribute meaningfully.
The gain over the initial engineering value of
$\lambda{=}2.0$ is $+1.42$ net utility in Non-IID,
confirming that the grid search yields a meaningful
improvement rather than a marginal one.

\begin{figure}[H]
\centering
\includegraphics[width=\linewidth]{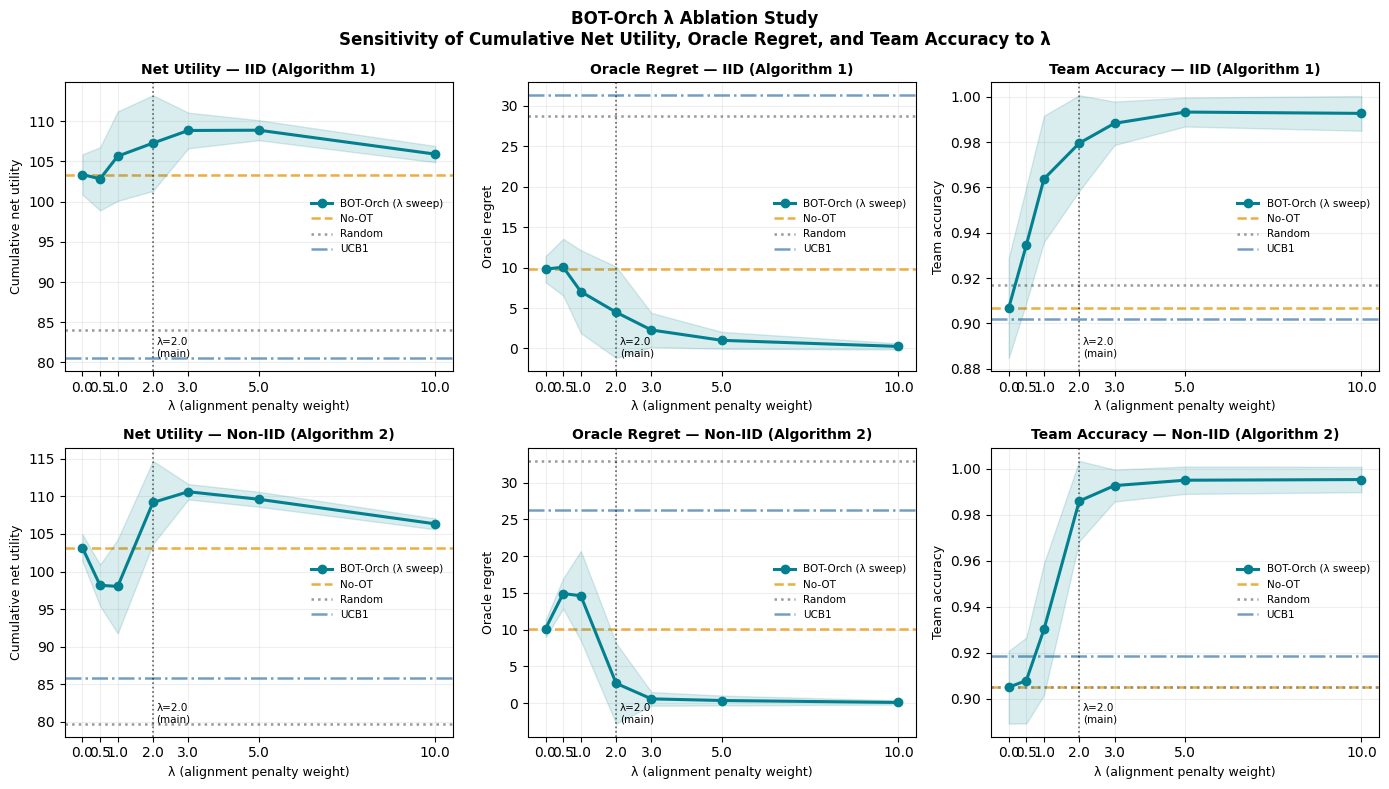}
\caption{%
  \textbf{$\lambda$ sensitivity curves.}
  Cumulative net utility (\textit{left}), oracle regret
  (\textit{centre}), and team accuracy (\textit{right})
  as functions of $\lambda$ for BOT-Orch (teal line with
  $\pm$1\,std shading).
  Horizontal dashed lines show the No-OT, Random, and
  UCB1 baselines.
  The dotted vertical line marks $\lambda{=}2.0$
  (initial value); $\lambda^*{=}3.0$ (grid-search
  optimum) is identified by the peak.
  \textit{Top row}: IID (Algorithm~1);
  \textit{bottom row}: Non-IID (Algorithm~2).
  In the Non-IID panels, performance at
  $\lambda\in\{0.5,1.0,1.25\}$ falls below the No-OT
  baseline, revealing the phase transition.}
\label{fig:lambda_ablation}
\end{figure}

\begin{figure}[H]
\centering
\includegraphics[width=0.88\linewidth]{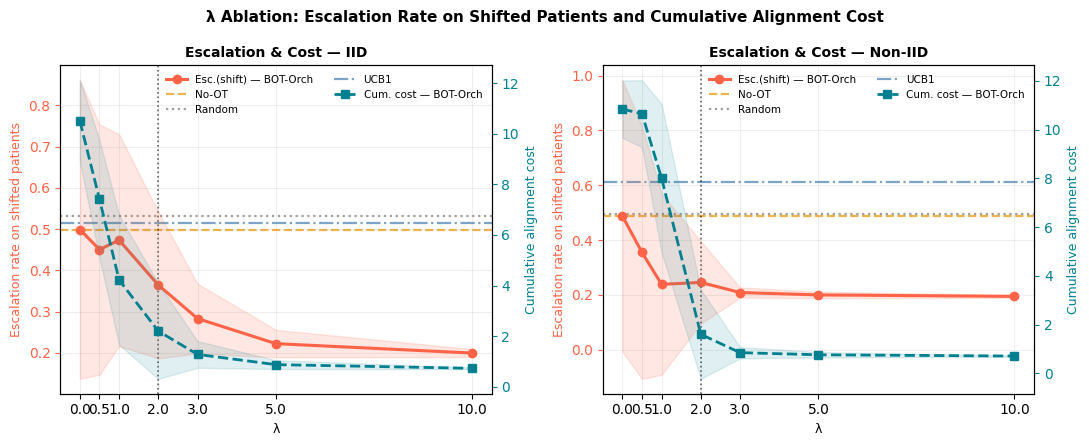}
\caption{%
  \textbf{Escalation rate on shifted patients (left axis,
  red) and cumulative alignment cost (right axis, teal)
  as functions of $\lambda$.}
  Horizontal dashed lines show No-OT, Random, and UCB1
  escalation rates.
  Alignment cost decreases monotonically with $\lambda$,
  confirming Theorem~4.2 Part~1.
  \textit{Left}: IID condition;
  \textit{right}: Non-IID condition.}
\label{fig:lambda_esc_cost}
\end{figure}

\end{document}